\pdfoutput=1
\documentclass[10pt,twocolumn,letterpaper]{article}

 \usepackage{cvpr}               

\usepackage{graphicx, epstopdf}
\usepackage{amsmath}
\usepackage{amssymb}
\usepackage{booktabs}
\usepackage{amsthm}
\usepackage{bbold}
\usepackage{mathtools}
\usepackage{algorithm}
\usepackage{algorithmic}
\usepackage{stackengine}
\usepackage[accsupp]{axessibility}  

\def\delequal{\mathrel{\ensurestackMath{\stackon[1pt]{=}{\scriptstyle\Delta}}}}
\newtheorem{theorem}{Theorem}
\newtheorem{proposition}{Proposition}
\newtheorem{lemma}[theorem]{Lemma}

\floatstyle{ruled}
\newfloat{listing}{tb}{lst}{}
\floatname{listing}{Listing}

\usepackage[pagebackref,breaklinks,colorlinks]{hyperref}

\usepackage[capitalize]{cleveref}
\crefname{section}{Sec.}{Secs.}
\Crefname{section}{Section}{Sections}
\Crefname{table}{Table}{Tables}
\crefname{table}{Tab.}{Tabs.}

\def\cvprPaperID{10438} 
\def\confName{CVPR}
\def\confYear{2022}

\usepackage{overpic}
\usepackage{enumitem} 
\usepackage{overpic} 
\usepackage{color}

\definecolor{turquoise}{cmyk}{0.65,0,0.1,0.3}
\definecolor{purple}{rgb}{0.65,0,0.65}
\definecolor{dark_green}{rgb}{0, 0.5, 0}
\definecolor{orange}{rgb}{0.8, 0.6, 0.2}
\definecolor{red}{rgb}{0.8, 0.2, 0.2}
\definecolor{darkred}{rgb}{0.6, 0.1, 0.05}
\definecolor{blueish}{rgb}{0.0, 0.3, .6}
\definecolor{light_gray}{rgb}{0.7, 0.7, .7}
\definecolor{pink}{rgb}{1, 0, 1}
\definecolor{greyblue}{rgb}{0.25, 0.25, 1}

\usepackage{blindtext}

\renewcommand{\paragraph}[1]{\vspace{1em}\noindent\textbf{#1}.}
\begin{document}

\title{Bayesian Nonparametric Submodular Video Partition for Robust\\ Anomaly Detection}
\author{ Hitesh Sapkota $\quad$  Qi Yu \\
Rochester Institute of Technology\\
{\tt\small \{hxs1943, qi.yu\}@rit.edu}
}
\maketitle

\begin{abstract}
\vspace{-2mm}
   Multiple-instance learning (MIL) provides an effective way to tackle the video anomaly detection problem by modeling it as a weakly supervised problem as the labels are usually only available at the video level while missing for frames due to expensive labeling cost. We propose to conduct novel Bayesian non-parametric submodular video partition (BN-SVP) to significantly improve MIL model training that can offer a highly reliable solution for robust anomaly detection in practical settings that include outlier segments or multiple types of abnormal events. BN-SVP essentially performs dynamic non-parametric hierarchical clustering with an enhanced self-transition that groups segments in a video into temporally consistent and semantically coherent hidden states that can be naturally interpreted as scenes. Each segment is assumed to be generated through a non-parametric mixture process that allows variations of segments within the same scenes to accommodate the dynamic and noisy nature of many real-world surveillance videos. The scene and mixture component assignment of BN-SVP also induces a pairwise similarity among segments, resulting in non-parametric construction of a submodular set function. Integrating this function with an MIL loss effectively exposes the model to a diverse set of potentially positive instances to improve its training. A greedy algorithm is developed to  optimize the submodular function and support efficient model training. Our theoretical analysis ensures a strong performance guarantee of the proposed algorithm. The effectiveness of the proposed approach is demonstrated over multiple real-world anomaly video datasets with robust detection performance.

\end{abstract}
\vspace{-4mm}
\section{Introduction}
\vspace{-2mm}
Anomaly detection from videos poses fundamental challenges as abnormal activities are usually rare, complicated, and unbounded in nature \cite{Liu2019MarginLE}. Furthermore, segment or frame labels are typically unavailable due to high labeling cost and therefore, the detection models have to rely on the weak video level labels \cite{Sultani2018}. There are two main streams of work to handle the challenging anomaly detection task. The first stream treats anomaly detection as an unsupervised learning problem \cite{TianZ0G19}. It assumes that an event is considered to be abnormal if it deviates significantly from a predefined set of normal events included in the training data \cite{Cong2011,Huan2015,  Vu2019RobustAD}. 
However, a model trained on limited normal data is likely to capture only specific characteristics present in the training dataset, and therefore, testing normal events deviating significantly from the training normal events will lead to a high false alarm rate \cite{Zhong2019}. The second stream of research has attempted to address the limitation by formulating the problem as multiple instance learning (MIL) that models each video as a bag and its segments (or frames) as instances within the bag \cite{Sultani2018}. 
The goal is to learn a model that can effectively make frame-level anomaly predictions relying on the video-level labels during the training process. One effective MIL learning objective is to maximize the gap between two instances having the respective highest anomaly scores from a pair of positive and negative bags. The maximum score based MIL (referred as MMIL) model outperformed the unsupervised approaches and achieved promising performance in real-world long surveillance videos \cite{Sultani2018}.

\begin{figure*}[t!]
\centering
\begin{subfigure}{0.19\textwidth}
  \centering
  \includegraphics[width=0.9\linewidth]{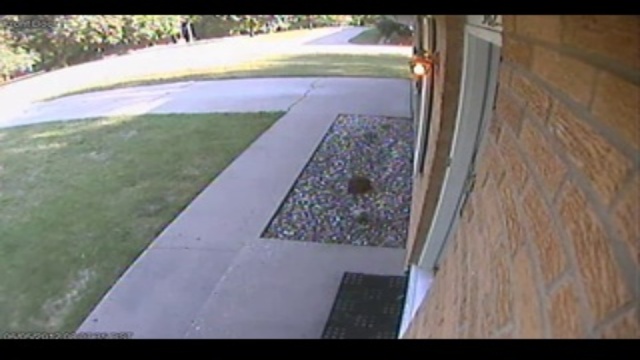}
  \caption{Burglary (Outlier 1)}
\end{subfigure}
\begin{subfigure}{0.19\textwidth}
  \centering
  \includegraphics[width=0.9\linewidth]{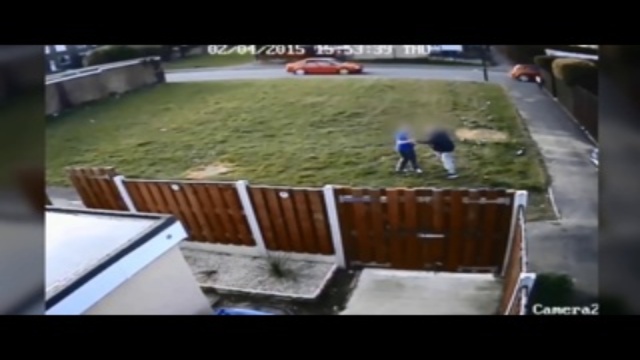}
  \caption{Walking (Outlier 2)}
\end{subfigure}
\begin{subfigure}{.19\textwidth}
  \centering
  \includegraphics[width=0.9\linewidth]{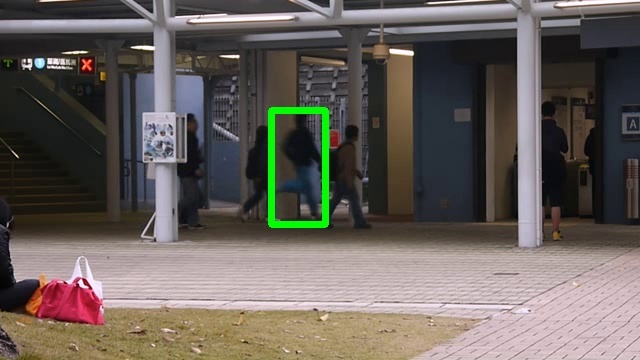}
  \caption{Running (Modality 1)}
\end{subfigure}%
\begin{subfigure}{0.19\textwidth}
  \centering
  \includegraphics[width=0.9\linewidth]{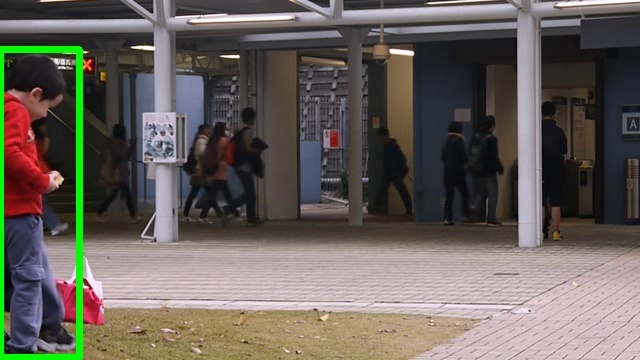}
  \caption{Loitering (Modality 2)}
\end{subfigure}
\begin{subfigure}{0.19\textwidth}
  \centering
  \includegraphics[width=0.9\linewidth]{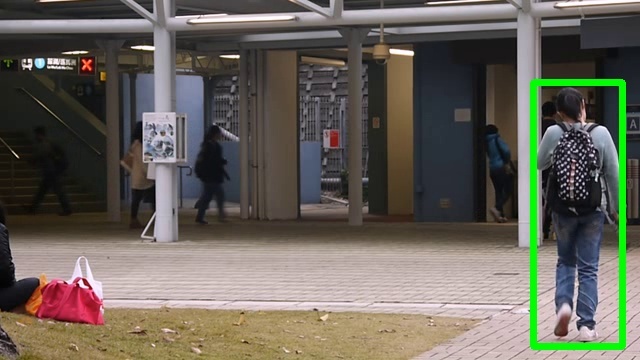}
  \caption{Walking (Modality 3)}
\end{subfigure}
\vspace{-2mm}
\caption{Examples of outlier (a-b) and multimodal frames (c-e) from the Avenue dataset}
\label{fig: different_anomaly_examples}
\vspace{-3mm}
\end{figure*} 
However, there are two key limitations with the MMIL model. First, the presence of noisy outliers (different from other normal events) in both abnormal and normal videos may significantly impact the overall model performance. This is because the objective function solely focuses on the individual segments from both positive and negative bags, making the training process sensitive to noises. Figure~\ref{fig: different_anomaly_examples} (a-b) shows the example normal frames that are significantly different from other normal ones in real-world surveillance videos. The first frame is from the burglary video that looks similar to an abnormal frame from a video with an arson event. The second frame is from the shooting video that looks similar to a fighting frame. Hence, they may serve as outliers in the corresponding videos. 

Second, if multiple types of abnormal events (referred to as multi-modal anomaly) present in a single abnormal video, MMIL may only detect one type of anomaly while missing other important ones due to the limitation of the objective function. Figure \ref{fig: different_anomaly_examples} (c)-(e) demonstrate three frames with different anomaly types from an example video in the Avenue dataset \cite{Lu2013}. In Figure \ref{fig: different_anomaly_examples} (c), the person is running, which is regarded as a strange action in that context~\cite{Lu2013}. In (b), it shows a person waiting in a place holding some object in the hand, and (c) involves a person walking in the wrong direction. Therefore, the single video has multiple anomaly frames leading to a multimodal scenario.

Top-$k$ ranking loss has been adopted in an attempt to address the issues as outlined above. It maximizes the gap between the mean score of the top-$k$ predicted instances from a positive bag and that of a negative one \cite{Tian_2021_ICCV,Sapkota2021}. However, there are inherent limitations by using a top-$k$ loss. First, it tends to be extremely sensitive to the selected $k$ value. Figure \ref{fig: sensitivity_auc_k} shows the highly fluctuating detection performance from two real-world surveillance video datasets. Since there is no frame (or segment) labels available during model training, setting an optimal $k$ through cross-validation is infeasible or highly costly. Furthermore, given the diverse videos, the number of abnormal instances may vary significantly from one video to another implying we should have a different $k$ for each video. Hence, applying the same $k$ to all videos as in the existing approaches fails to capture the nature of the data. Another serious but more subtle issue is that all (or most of) the selected $k$ segments may come from the same sub-sequence of the video. Using a consecutive set of visually similar segments is less effective for model training, making it more likely to suffer from outlier and multimodal scenarios. As a result, top-$k$ approaches will fall short in providing a reliable detection performance in most practical settings as evidenced by our experiments.


\begin{figure}[t!]
\vspace{-10mm}
\centering
\begin{subfigure}{0.23\textwidth}
  \centering
  \includegraphics[width=\linewidth]{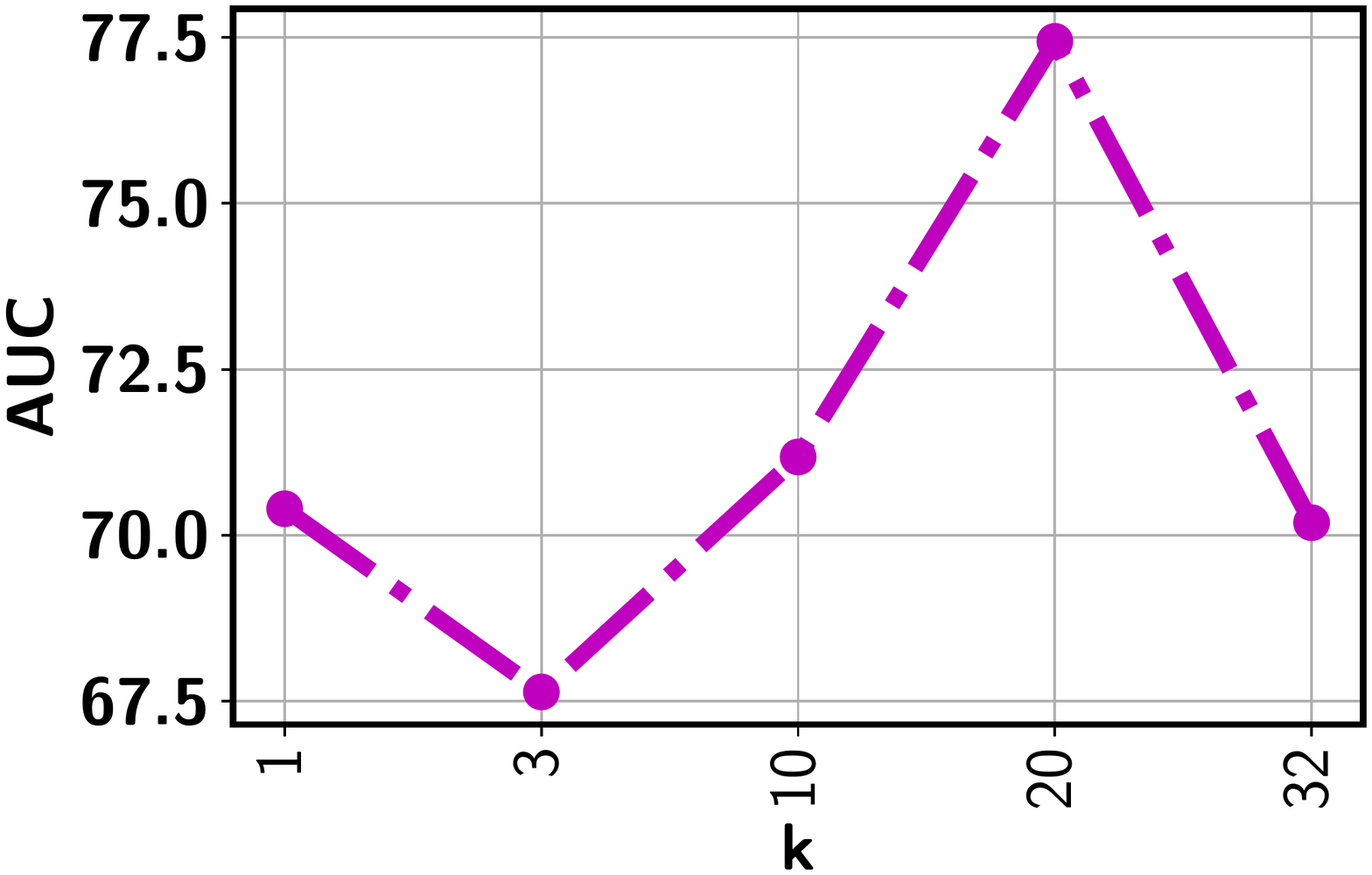}
  \vspace{-18mm}
  \caption{Avenue}
\end{subfigure}
\begin{subfigure}{0.23\textwidth}
  \centering
  \includegraphics[width=\linewidth]{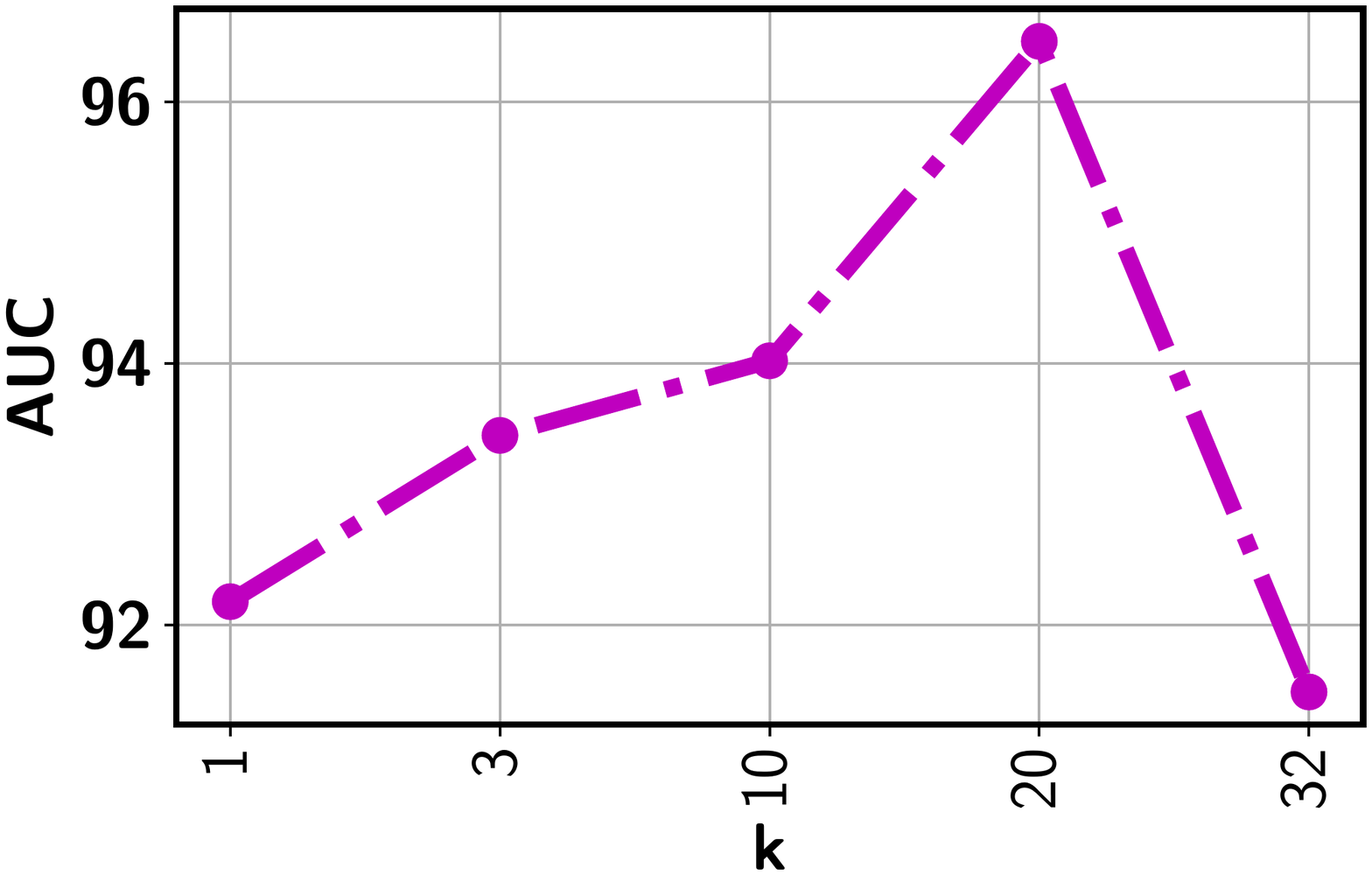}
  \vspace{-18mm}
  \caption{ShanghaiTech}
\end{subfigure}
\vspace{-2mm}
\caption{Highly fluctuating detection performance w.r.t. $k$}
\label{fig: sensitivity_auc_k}
\vspace{-6mm}
\end{figure}

To address the fundamental limitations of existing solutions, we propose novel Bayesian non-parametric construction of a submodular set function, which is integrated with multiple instance learning to deliver robust video anomaly detection performance under practical settings. Instead of choosing a set of instances with the highest prediction scores that are likely from a consecutive sub-sequence, maximizing a specially designed submodular function can involve a more diverse set of instances and expose the model to all potentially abnormal segments for more effective model training. Furthermore, the submodular set function is constructed in a non-parametric way, which induces a pairwise similarity among different segments in a video based on the diverse nature of the data. More specifically, an infinite Hidden Markov Model with a Hierarchical Dirichlet prior (HDP-HMM) \cite{teh2006hierarchical} augmented with an enhanced self-transition is employed to partition a video through dynamic non-parametric clustering of its segments. To more effectively accommodate the dynamic and noisy nature of real-world surveillance videos, the emission process of the HMM is also governed by a non-parametric mixture model to allow segments within the same hidden state to have visual and spatial variations. This unique design is instrumental to discover temporally consistent and semantically coherent hidden states that can be naturally interpreted as scenes. Pairwise similarity among different segments is defined according to the state-component structure, which leads to the construction of a submodular set function. We then develop a novel submodularity diversified MIL loss function to ensure robust anomaly detection from real-world surveillance videos with outlier and multimodal scenarios. Our key contributions include:
\begin{itemize}[noitemsep,topsep=0pt,leftmargin=*]
    \item Formulation of a novel {\bf submodularity diversified MIL loss} that simultaneously extracts a diverse set of potentially positive instances while maximizing the gap between the mean score of these instances from a positive bag and a negative one, respectively. 
    \item {\bf Bayesian non-parametric construction of the submodular set function} that  infers the diversity from the video data to induce a pairwise similarity among different segments in a video and provide an upper bound on the size of the diverse set.
    \item {\bf A greedy algorithm} that leverages the state-component hierarchical structure resulting from the non-parametric construction for submodular set function optimization and efficient model training.
    \item {\bf Theoretical results} to ensure strong performance guarantee of the greedy algorithm. 
\end{itemize}
The proposed approach achieves the state-of-the-art robust anomaly detection performance on real-world surveillance videos with noisy and multimodal scenarios.
\vspace{-4mm}
\section{Related Work}
\label{sec: related work}
\vspace{-2mm}
Encoding and sparse reconstruction-based approaches have been employed for anomaly detection, assuming that abnormal events are rare and deviate from normal patterns. They aim to capture the normal patterns using models, such as Gaussian processes (GPs)~\cite{Li2015AnomalyDI} and HMMs \cite{Kratz2009}, to identify anomalies as outliers based on the reconstruction loss. Sparse representation-based approaches  construct a dictionary for normal events and identify the events with the high reconstruction error as anomalies \cite{Lu2013}. Recent approaches consider both abnormal and normal events in the training process. For video anomaly detection, since only video-level labels are assumed to be available during model training \cite{He2018}, MIL offers a natural solution by modeling each video as a bag and the associated segments (frames) as instances of the bag. Sultani et al. proposed an MIL based approach that enables to maximize the gap between highest prediction scores from a positive and negative bags, respectively \cite{Sultani2018} . However, this maximum score based MIL model (\ie, MMIL) is insufficient to handle outlier and multimodal scenarios as discussed earlier. 


Top-$k$ ranking loss based MIL models have been developed to address the limitations of the MMIL model~\cite{Tian_2021_ICCV,Sapkota2021}. These models produce state-of-the-art detection performance given that a suitable $k$ value can be assigned in advance. However, as demonstrated earlier, the detection performance of such models is highly sensitive to the chosen $k$ value. Meanwhile, given the diverse nature of videos, applying the same $k$ value to all the videos is sub-optimal. More importantly, since instance level labels are not available during training time, choosing a single $k$ value through cross-validation is infeasible or incurs a high annotation cost. Distributionally Robust Optimization (DRO) has been used to convert the top-$k$ set into an uncertainty set that allows the model to focus on instances in proportion to their prediction scores \cite{Sapkota2021}. This is equivalent to assigning soft membership to involve instances into the MIL loss function. However, the size of the uncertainty set is controlled by the radius (\ie, $\eta$) of the uncertainty ball, which needs to be manually set. Furthermore, the model may put more focus on a set of consecutive segments with the highest prediction scores and ignore some other potentially positive segments. 

The proposed approach constructs a novel submodular set function in a non-parametric way by inferring the diversity from data automatically. By jointly optimizing the submodular function and the MIL loss, it automatically chooses a diverse set of segments and lets the model better differentiate these (potentially positive) segments from those of a negative bag to ensure good detection performance.

\section{Methodology}
\vspace{-1mm}
Following the standard MIL assumption, we consider, for a positive bag, there is at least one abnormal segment whereas, for a negative bag all segments are of normal types. Table \ref{tab: symbol_table} in the Appendix summarizes the major symbols and their descriptions.

\begin{figure}[t!]
\vspace{-20mm}
\centering
  \includegraphics[width=0.7\columnwidth]{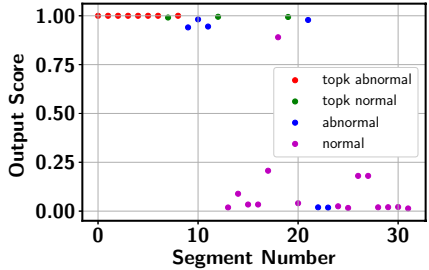}
 \vspace{-20mm}
\caption{Output of the top-$k$ based approach in a video from Avenue datase (missing some of the abnormal segments in top-$k$).}
\vspace{-5mm}
\label{fig: drawback_topk}
\end{figure}

\subsection{Preliminaries}
Let ${\bf x}_i^+$ be the $i^{th}$ segment in the positive bag $\mathcal{B}_{pos}$ and ${\bf x}_j^-$ indicates the $j^{th}$ segment in the negative bag $\mathcal{B}_{neg}$. Also consider $n$ as the total number of instances per bag. The maximum score based MIL (MMIL) model tries to maximize the gap between the maximum prediction score from positive bag and that from the negative bag \cite{Sultani2018}:
\begin{equation}
    L(\mathcal{B}_{pos}, \mathcal{B}_{neg}) = \Bigl[1-\max_{i\in \mathcal{B}_{pos}}(f({\bf x}_i^+))+\max_{j\in \mathcal{B}_{neg}}(f({\bf x}_j^-))\Bigr]_+\label{eq:max_mil}
\end{equation}
where $f({\bf x}_i^+)$ (or $f({\bf x}_j^-)$) is the prediction score of ${\bf x}_i^+$ (or ${\bf x}_j^-$) and $[a]_+ = \max\{0, a\}$. 
As mentioned earlier, MMIL is less effective to handle outlier and multimodal scenarios. The top-$k$ ranking loss partially addresses the limitation of MMIL by maximizing the gap between an average of $k$ highest segment predictions  from the positive bag and maximum segment prediction score from a negative bag:  
\begin{align}
   \hspace*{-2.1mm} L(\mathcal{B}_{pos}, \mathcal{B}_{neg}) \hspace*{-.5mm}=\hspace*{-.5mm} \Bigl[1-\frac{1}{k}\sum_{i=1}^kf({\bf x}_{[i]}^+)+\max_{j\in \mathcal{B}_{neg}} f({\bf x}_j^-)\Bigl]_+\label{eq:avg_topk_mil}
\end{align}
where the positive bag segment predictions are sorted in a non-decreasing order, \ie, $f({\bf x}_{[1]}^+)\geq...\geq f({\bf x}_{[k]}^+)$.

\begin{figure*}[t!]
\centering
\begin{subfigure}{0.19\textwidth}
  \centering
  \includegraphics[width=0.9\linewidth]{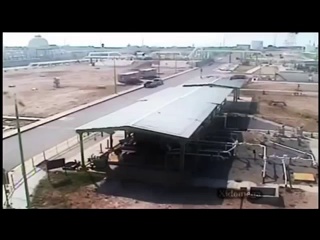}
  \caption{Normal}
\end{subfigure}%
\begin{subfigure}{0.19\textwidth}
  \centering
  \includegraphics[width=0.9\linewidth]{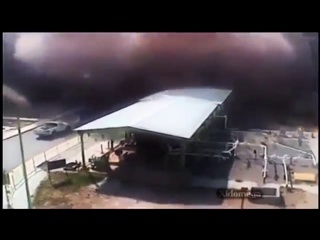}
  \caption{Abnormal}
\end{subfigure}%
\begin{subfigure}{.19\textwidth}
  \centering
  \includegraphics[width=0.9\linewidth]{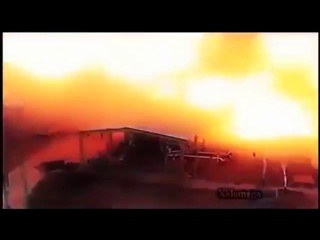}
  \caption{Abnormal}
\end{subfigure}%
\begin{subfigure}{.19\textwidth}
  \centering
  \includegraphics[width=0.9\linewidth]{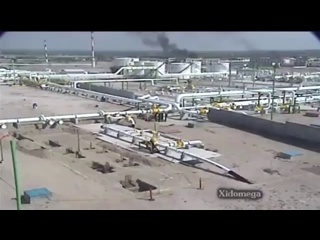}
  \caption{Normal}
\end{subfigure}%
\begin{subfigure}{0.19\textwidth}
  \centering
  \includegraphics[width=0.9\linewidth]{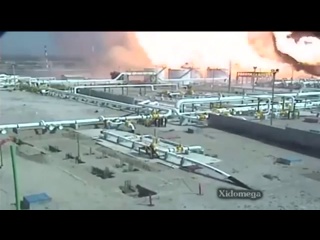}
  \caption{Abnormal}
\end{subfigure}
\vspace{-2mm}
\caption{Example frames from different scenes in an explosion video from UCF-Crime: (a-b) scene 1, (c) scene 2, (d-e) scene 3}
\label{fig: different_scenes}
\vspace{-4mm}
\end{figure*}

There are two major issues associated with the top-$k$ ranking loss. First, choosing an optimal $k$ value is a key challenge as the number of abnormal instances may vary significantly from one video to another implying different $k$ value for each video. Second, all the selected top-$k$ instances may come from same sub-sequence of the video. Including all those visually similar instances does not contribute much in the model training process. Furthermore, concentrating only on a specific sub-sequence may make the approach less effective to handle multimodal and outlier scenario. Figure \ref{fig: drawback_topk} presents the output of the top-$k$ based model in the Avenue dataset. It can be seen that the top-$k$ based approach picks the consecutive video segments while missing quite a few other abnormal frames.

\subsection{Bayesian Non-parametric Submodular Set Function Construction}
The proposed Bayesian non-parametric submodular video partition (BN-SVP) approach offers a novel integrated solution to  address the above two fundamental challenges simultaneously. In particular, since submodular set functions provide a natural measure for diversity, we design a special submodular set function that enables discovery of a representative set of segments from a video. This avoids only choosing visually similar consecutive video segments like in the top-$k$ approach, which enhances the model's exposure to potential abnormal instances during model training. As a result, the model's capability to handle multimodal and outlier scenarios can be effectively improved. 

However, maximizing a submodular set function still requires to specify the size of the set. As mentioned above, choosing a set with an optimal size in video anomaly detection is highly challenging. To this end, we propose a novel Bayesian non-parametric construction of the submodular set function. The non-parametric construction leverages both visual features of the video segments and their temporal closeness to derive a similarity measure that allows us to define a submodular set function $F (\mathcal{C}^+)$, where $\mathcal{C}^+$ represents a subset of segments in a video. The size of $\mathcal{C}^+$ is automatically determined through Bayesian non-parametric analysis of the video. Intuitively, most videos, especially those with anomalies, usually consist of multiple scenes, where each scene is comprised of a consecutive set of visually similar segments. Figure~\ref{fig: different_scenes} shows the example frames from three different scenes in a video that records an explosion event. Ideally, if a video could be partitioned based on these scenes, we can choose representative (and potentially positive) segments from each scene. Such information can significantly facilitate the optimization of the submodular set function. However, both the number and the types of the scenes are unavailable during model training.  

The proposed BN-SVP addresses the above issue through non-parametric partition of a video. It  builds upon and extends an HDP-HMM model that places a  Hierarchical Dirichlet Process (HDP) prior on the state transition distribution of a Hidden Markov Model (HMM) model~\cite{teh2006hierarchical}. By using an HMM to model a video (as a sequence of segments), each discovered hidden state can be naturally interpreted as a scene in the video. The HDP prior allows us to determine the optimal number of states (\ie, scenes) according to the nature of the data. However, real-world videos may be highly noisy and directly using an HDP-HMM model may extract too many scenes with less significant visual characteristics (\eg, spatial changes of objects or addition/removal of a small number of objects). To address this issue, we follow the sticky HDP-HMM to encourage a stronger self-transition of a state~\cite{Fox2011ASH}. This will result in temporal persistence of states to produce longer and semantically coherent scenes. To further accommodate spatial changes or variations in certain objects, we allow the emission distribution to follow  another non-parametric DP that automatically determines the number of mixture components (\ie, sub-scenes) within the same scene. For example, scene 1 in Figure~\ref{fig: different_scenes} is comprised of two sub-scenes: the first with a clear sky and the second with smoke in the sky. 


More specifically, consider a collection of hidden states (\ie scenes in a video), the transition probability of state $j$ to other states is governed by a DP:
\begin{equation}
    \mathcal{G}_j = \sum_{l=1}^\infty \hat{\pi}_{jl}\delta_{\hat{\phi}_{jl}}, \; \hat{\boldsymbol{\pi}}_j\thicksim \text{GEM}(\alpha)
\end{equation}
where $\text{GEM}(\alpha)$ is formed through a stick breaking construction process with parameter $\alpha$~\cite{teh2006hierarchical}, $\hat{\phi}_{jl}$ is drawn from a base distribution $\mathcal{G}_0$, which follows another DP
\begin{equation}
    \mathcal{G}_0 = \sum_{k=1}^\infty\beta_k \delta_{\phi_k}, \; \boldsymbol{\beta} \thicksim \text{GEM}(\gamma), \; \phi_k\thicksim H
\end{equation}
Because of the discrete nature of $\mathcal{G}_0$, there can be multiple $\hat{\phi}_{jl}$'s taking an identical value of $\phi_k$. Considering the unique set of atoms $\phi_k$, we can rewrite $\mathcal{G}_j$ as
\begin{equation}
    \mathcal{G}_j = \sum_{k=1}^\infty \pi_{jk}\delta_{\phi_k}, \; \boldsymbol{\pi}_j\thicksim \text{DP} (\alpha, \boldsymbol{\beta}), \phi_k\thicksim H
\end{equation}
Given the highly dynamic and noisy nature of many real-world surveillance videos, directly applying the standard HDP-HMM model to partition a video may result in many redundant scenes and rapidly switches among them. This is problematic in our setting in which it is critical to infer semantically coherent scenes along with a slower transition among them. As a result, it is essential to ensure temporal persistence of the discovered scenes \cite{Fox2011ASH}. This can be achieved through enhanced self transitions. In particular, the transition probability of the $j$'s state is augmented by 
\begin{equation}
    \boldsymbol{\pi}_j \thicksim \text{DP}\left(\alpha+\rho, \frac{\alpha\boldsymbol{\beta}+\rho \delta_j}{\alpha+\rho}\right)
\end{equation}
This has the effect of increasing the expected probability of staying in the same state. 
\begin{equation}
 \mathbb{E}[\pi_{jk}|\boldsymbol{\beta}] = \left\{
  \begin{array}{@{}ll@{}}
    \frac{\alpha\beta_j+\rho}{\alpha+\rho} & \text{if}\  k=j \\
  \frac{\alpha\beta_k}{\alpha+\rho}, & \text{otherwise}
  \end{array}\right.
\label{eq:expected_transition_probability}
\end{equation}
To allow certain levels of variations within the same scene and accommodate the highly dynamic nature of a video sequence, we propose to model the emission process using a mixture distribution governed by another non-parametric DP. This design offers three unique advantages. First, it further ensures the temporal persistence of a scene as for a segment with less significant visual differences, it can stay in the same scene by switching to a different mixture component instead of transitioning to another (redundant) scene. Second,  it offers a fine-grained partition of the video sequence, which is instrumental to separate abnormal segments (\eg, frames (b) and (e) in Figure~\ref{fig: different_scenes}) from normal ones (\eg, frames (a) and (d) in Figure~\ref{fig: different_scenes}) that share a common background. Last, the number of mixture components is automatically determined by the DP (\eg, scenes 1 \& 3 have two mixture components while scene 2 only has one).  
For the $k$-th scene, there is an unique stick-breaking distribution $\boldsymbol{\psi}_k\thicksim \text{GEM}(\tau)$ that defines the weights of the mixture components within the scene. Then, given the scene and mixture component assignment $(z_i=k, s_i=t)$ of a segment ${\bf x}^+_i$ in a video, it is drawn from a specific multivariate Gaussian: $\mathcal{N}(\boldsymbol{\mu}_{k,t}, \Sigma_{k,t})$. 

Posterior inference of the augmented HDP-HMM model with a DP mixture for emission can be achieved through direct assignment \cite{teh2006hierarchical} or blocked sampling with an increasing mixing rate \cite{Fan2017}. Hyper-parameters can also be inferred by placing a vague prior on them and conduct Gibbs sampling. 

The scene and component assignments of BN-SVP induces a pairwise similarity among segments in a video:
\begin{equation}
\label{eq:similarity}
  \left\{
  \begin{array}{@{}ll@{}}
      S_{i, j} = ({\bf x}_i^+)^\top\Sigma^{-1}_{z_i,s_i}{\bf x}_j^+ & \text{if}\ s_i==s_j \wedge z_i==z_j \\
      S_{i, j} = 0 & \text{otherwise}
  \end{array}\right.
\end{equation}
It is worth to note that the similarity between two segments ${\bf x}_i^+$ and ${\bf x}_j^+$ is evaluated using the learned feature representations (through a DNN) instead of the raw features. The induced similarity allows us to define a submodular set function \cite{Krause2008submodular, Krause08optimizingsensing} summarized by the follow proposition. 
\begin{proposition}
Let $\kappa$ denote the number of unique mixture components across all the discovered states in a bag $\mathcal{B}_{pos}$ and $\mathcal{C} \subset \mathcal{B}_{pos}$ is a subset of $\mathcal{B}_{pos}$ with size $\kappa$. Given the BN-SVP induced pairwise similarity defined in \eqref{eq:similarity}, the following function is a submodular set function:
\begin{equation}
F({\mathcal C}) = \sum_{i\in {\mathcal{B}_{pos}}}\max_{j \in {\mathcal C}} S_{i, j}
    \label{eq: submodular_function}
\end{equation}
\end{proposition}
Based on the definition of $S_{i,j}$ as shown above, it is straightforward to show that $F({\mathcal C})$ is a special instance of the location facility function~\cite{lin2009select}, which is submodular. Since each mixture component captures a unique sub-scene, maximization of $F({\mathcal C})$ can extract a diverse set of   segments that best represent the all the scenes (and sub-scenes) in the entire video. 
By further integrating the margin loss given in \eqref{eq:avg_topk_mil}, we achieve a submodularity diversified MIL loss: 
\begin{align}
    &\min_{{\bf w}, \mathcal{C}^+\in \mathcal{B}_{pos}, |\mathcal{C}^+|\leq {\kappa}}L(\mathcal{C}^+)-\lambda F(\mathcal{C}^+)
    \label{eq: sub_cons_avg_topk}
\end{align}
where the margin loss is defined over instances in a set $\mathcal{C}^+$ with size no larger than $\kappa$:
\begin{equation}
    L(\mathcal{C}^+)=\Bigl[1-\frac{1}{|\mathcal{C}^+|}\sum_{i\in \mathcal{C}^+}f({\bf x}^+_i)+\max_{j\in \mathcal{B}_{neg}} f({\bf x}_j^-)\Bigl]_+
    \label{eq: pairwise_loss}
\end{equation}
In essence, $\mathcal{C}^+$ includes the set of instances in a positive bag that participate in the model training. The constraint $|\mathcal{C}^+|\leq {\kappa}$ has the effect of excluding some representative segments from the margin loss as these segments are less likely to be abnormal (\eg, with a very low predicted score). Including these segments will increase the margin loss and coefficient $\lambda$ controls the balance between the margin loss and the diversity among the chosen segments.

\subsection{Greedy Submodular Function Optimization}
We propose a greedy algorithm for optimizing the submodular function in \eqref{eq: submodular_function} to ensure efficient model training. The proposed algorithm leverages the special structure of the state and mixture component space resulted from the HDP-HMM partition of the video segments. The performance guarantee of the greedy algorithm is ensured by our theoretical result presented at the end of this section. 

Recall that we use $s_i$ to denote the mixture component of segment ${\bf x}^+_i$. Let $f^*_s$ denote the maximum score among all the segments assigned to the same mixture component and $i^*_s$ be the index of the corresponding representative segment:
\begin{align}
    i^*_s = \arg\max_{\forall i: s_i=s}f({\bf x}_i^+),\;
    f_{s}^* = f({\bf x}_{i^*_s}^+)
\end{align}
We construct a representative set $\widehat{\mathcal{C}^+}$ as follows. Let $\widehat{\mathcal{C}^+}=\Phi$ and for each mixture component $s$, we set
\begin{equation}
  \left\{
  \begin{array}{@{}ll@{}}
      \widehat{\mathcal{C}^+}\leftarrow \widehat{\mathcal{C}^+} \cup \{i^*_s\}, & \text{if}\   f^*_{s} \geq \epsilon \\
    \widehat{\mathcal{C}^+}\leftarrow \widehat{\mathcal{C}^+}, & \text{otherwise}
  \end{array}\right.
 \label{eq:score-selection}
\end{equation}
where $\epsilon$ is a threshold to exclude segments with a low prediction score, which plays an equivalent role as constraint $|\mathcal{C}^+|\leq {\kappa}$ in \eqref{eq: sub_cons_avg_topk}.  In our experiments, we use $\epsilon$ equal to the output of the segment staying in the 35th percentile among video specific segments so as  to avoid skipping any potential abnormal segments. Once a representative set $\widehat{\mathcal{C}^+}$ is constructed, model training can proceed by solving the following MIL loss:
\begin{align}
   \min_{{\bf w}} \Bigl[1-\frac{1}{|\widehat{\mathcal{C}^+}|}\sum_{i\in \widehat{\mathcal{C}^+}}f({\bf x}^+_i)+\max_{j\in \mathcal{B}_{neg}} f({\bf x}_j^-)\Bigl]_+\label{eq:representative_topk_mil}
   \vspace{-2mm}
\end{align}
Given the state and mixture component assignment of each segment in a video, the representative set can be quickly constructed by sorting segments within each component according to their predicted scores and choosing the representative segment from each component by comparing its score with the threshold $\epsilon$. Next, we provide a strong theoretical guarantee that the greedy algorithm can ensure the inclusion of a diverse set of segments for model training. 
\begin{theorem}
The representative set based MIL loss given in \eqref{eq:representative_topk_mil} is  equivalent to the submodularity diversified MIL loss given in \eqref{eq: sub_cons_avg_topk}. Furthermore, using the proposed greedy algorithm to locate the $\kappa$ representative segments essentially provides a $\kappa$-constrained greedy approximation to the maximization of the submodular set function $F(\mathcal{C})$. As a result, the obtained solution is guaranteed to be no worse $(1-e^{-1})$ of the optimal solution.

\end{theorem}
The detailed proof is provided in the Appendix.
\vspace{-2mm}
\section{Experiments}
\vspace{-1mm}
We conduct extensive experiments to evaluate the effectiveness of the proposed BN-SVP approach. Through these experiments, we aim to demonstrate: (i) outstanding anomaly detection performance by comparing with competitive top-$k$, MIL, and other video anomaly detection models, (ii) robustness to outlier and multimodal scenarios, and (iii) deeper insights on the better detection performance through a qualitative study. 


\subsection{Datasets and Experimental Settings}
Our experimentation includes three video datasets of different scales: ShanghaiTech \cite{Luo2017}, Avenue \cite{Lu2013}, and UCF-Crime \cite{Sultani2018}. Table \ref{tab: video_level_distribution} in the Appendix shows how the videos are partitioned into the training/testing sets in each dataset. 
\begin{itemize}[noitemsep,topsep=0pt,leftmargin=*]
    \item {\bf ShanghaiTech} consists of 437 videos (330 normal and 107 abnormal). In the original setting, all training videos are normal. To fit into our setting, we follow the data split in \cite{Zhong2019} to assign normal and abnormal videos in both training and testing sets. 
    \item {\bf Avenue} consists of 16 training and 21 testing videos. 
We perform 80:20 split separately in the abnormal and normal video sets to generate training and testing instances.
\item {\bf UCF-Crime} consists of 13 different anomalies with a total of 1900 videos, where 1610 are training videos and 290 are testing videos. In this dataset, frame labels are available only for the testing videos. 
\end{itemize}
To show the robustness of the proposed approach in the multimodal and outlier scenarios, we also generate the Multimodal and Outlier datasets. Specifically, we create a multimodal scenario by extending the UCF-Crime dataset. For the outlier scenario, we deliberately impose some outliers in the ShanghaiTech dataset. More details of these two datasets are provided in Section \ref{sec:multimodal}. For evaluation, we report the frame-level receiver operating characteristics (ROC) curve along with the corresponding area under curve (AUC). The AUC score indicates the robustness of the performance at various thresholds.

\begin{figure*}[t!]
\centering
\vspace{-3mm}
\begin{subfigure}{.19\textwidth}
  \centering
  \includegraphics[angle=90, width=1\linewidth]{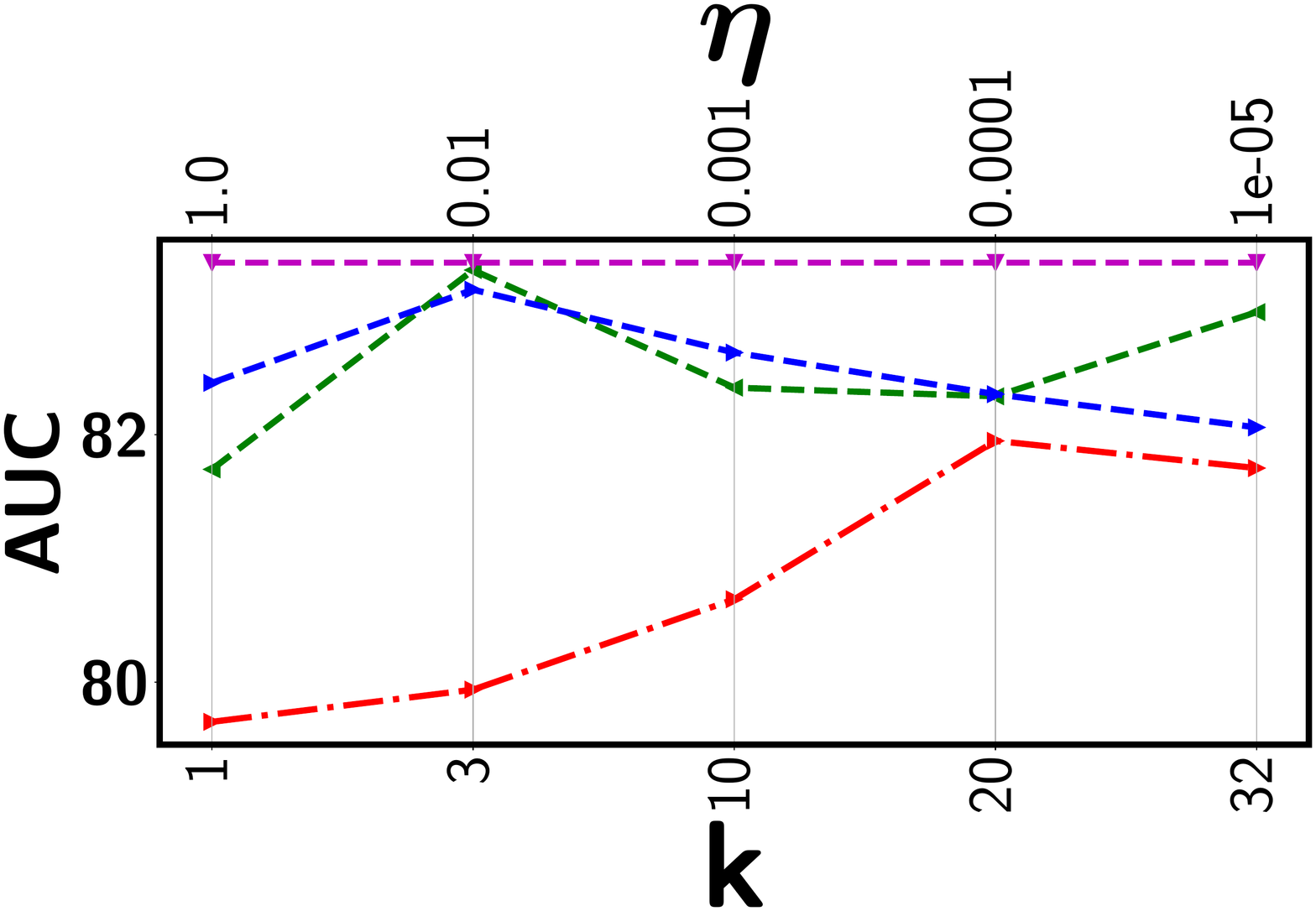}
  \caption{UCF Crime}
\end{subfigure}%
\begin{subfigure}{0.19\textwidth}
  \centering
  \includegraphics[angle=90,width=1\linewidth]{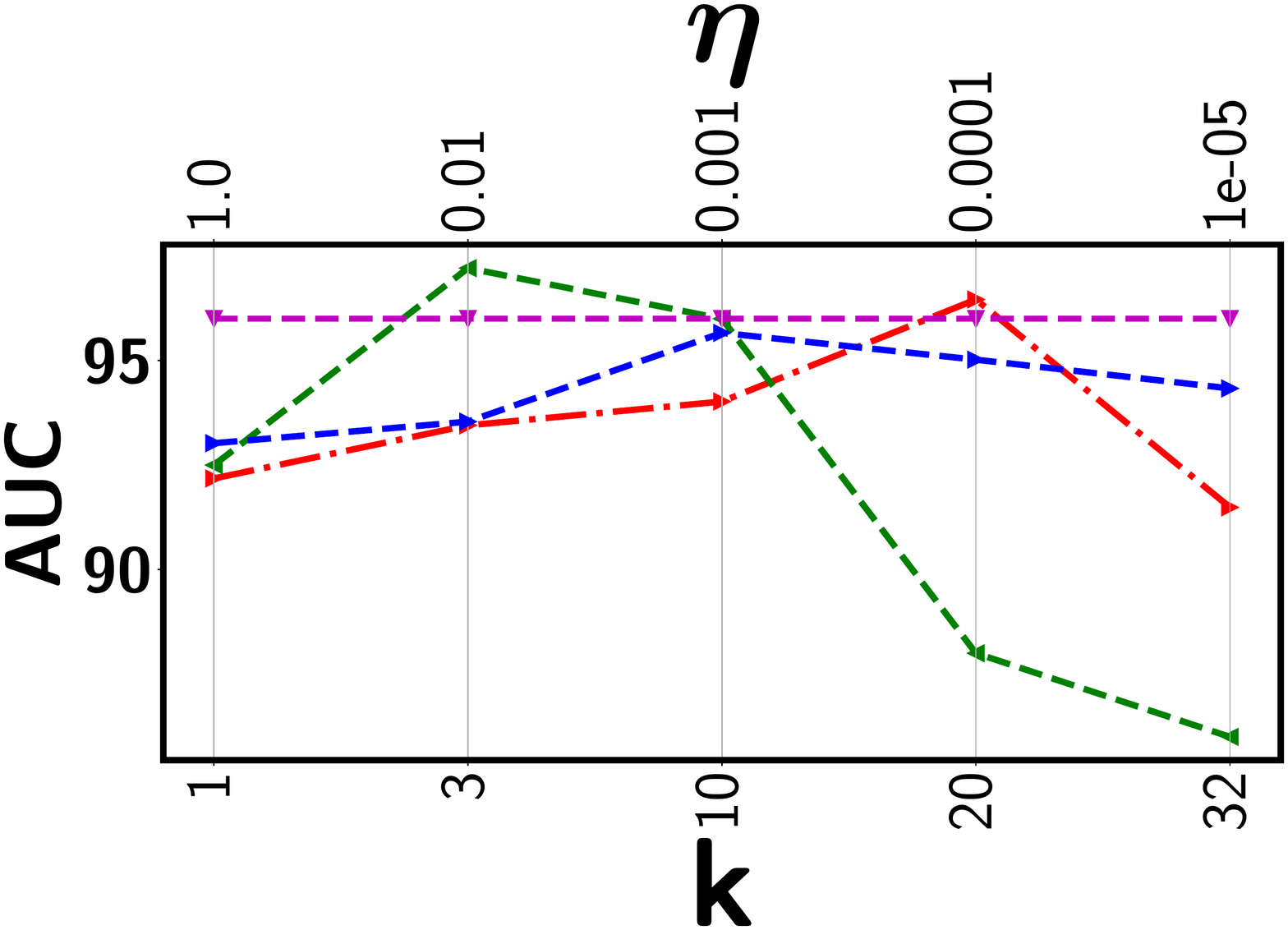}
  \caption{Shanghaitech}
\end{subfigure}%
\begin{subfigure}{0.19\textwidth}
  \centering
  \includegraphics[angle=90,width=1\linewidth]{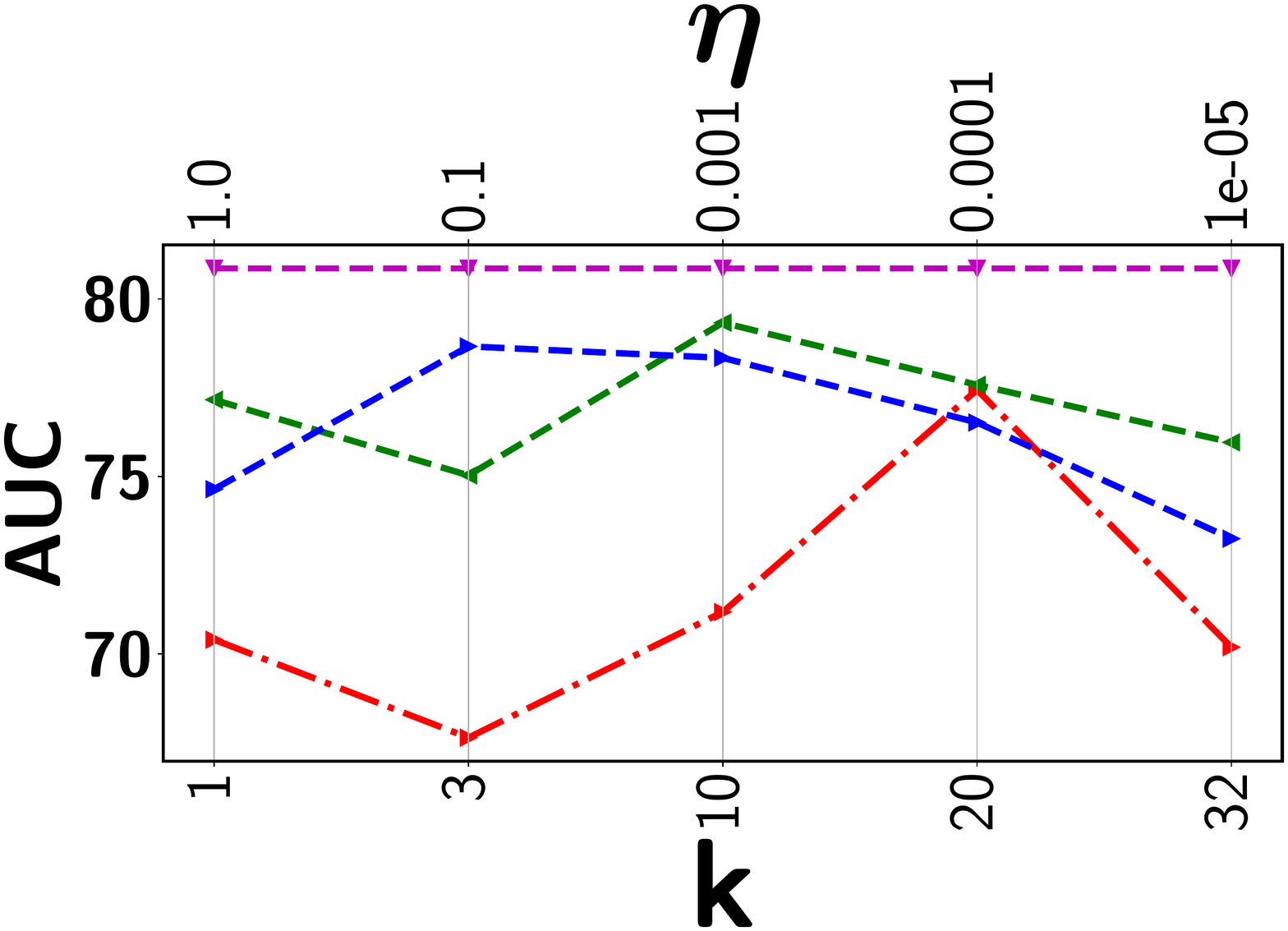}
  \caption{Avenue}
\end{subfigure}%
\begin{subfigure}{0.19\textwidth}
  \centering
  \includegraphics[angle=90,width=1\linewidth]{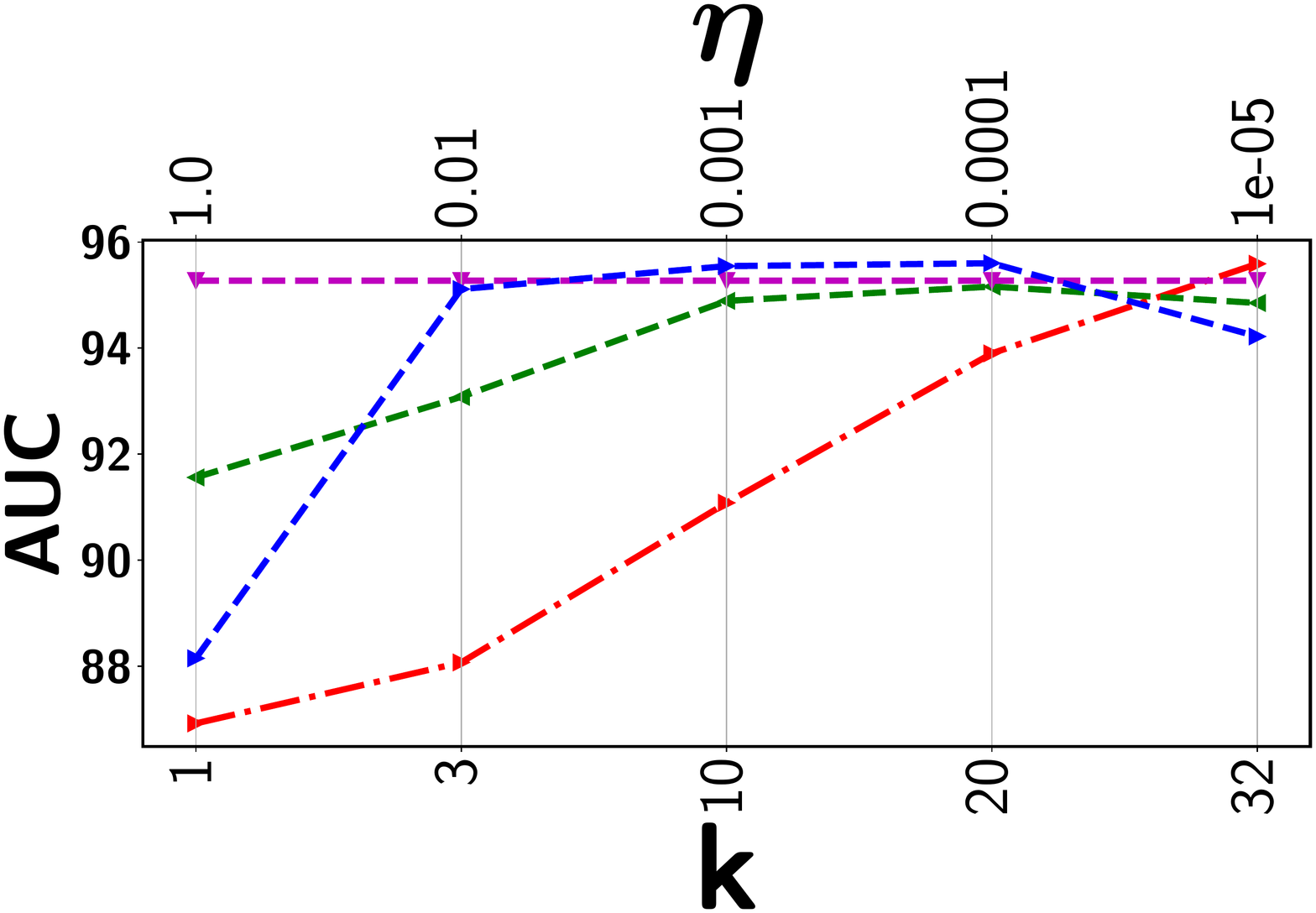}
  \caption{Outlier}
\end{subfigure}
\begin{subfigure}{0.19\textwidth}
  \centering
  \includegraphics[angle=90,width=1\linewidth]{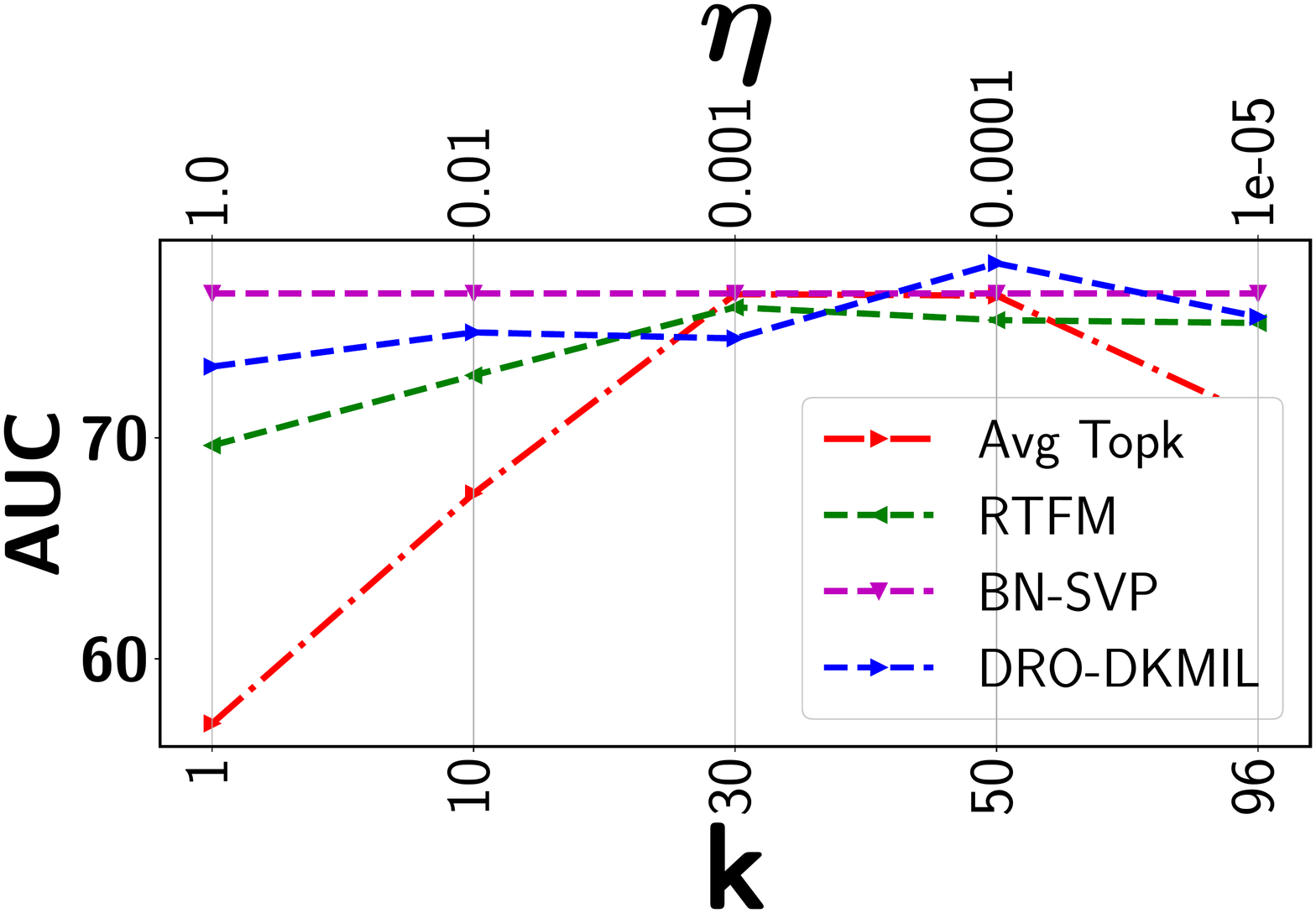}
  \caption{Multimodal}
\end{subfigure}
\vspace{-2mm}
\caption{Performance comparison with top-$k$ ranking models}
\label{fig: auc_k_eta}
\vspace{-5mm}
\end{figure*}

For Avenue and ShanghaiTech datasets, we extract visual features from the FC7 layer of a pre-trained C3D network \cite{Tran2015}. We re-size each video frame to $240\times 340$ pixels and fix the frame rate to 30 fps. We compute the C3D features for every 16-frame video clip. This may yield a different number of clips (each clip having a 2048 dimensional feature vector) depending on the number of frames in each video. Thus, we fit any number of clips to the 32 segments by taking an average of clip features in a specific segment.  In case of UCF-Crime, we extract the features using an I3D network \cite{Carreira2017} by using the pretrained network as described in \cite{Wang2018NonlocalNN}. For all datasets, we use parallel GCN  networks to capture the feature similarity and temporal consistency. The outputs of the parallel branches are combined and passed through a 5-layer LSTM network where each layer has 32 hidden units followed by batch normalization. Finally, an FC layer with sigmoid activation is applied to bring the prediction score to $(0,1)$. For model training, we use SGD with a learning rate of $0.001$ and $l_2$ regularization with parameter $\lambda = 0.001$. Detailed information about the network architecture is provided in the Appendix.

\subsection{Performance Comparison}
\vspace{-4mm}\paragraph{Comparison with Top-$k$ Models} We first compare the detection performance with two most recent top-$k$ based models, including Robust Temporal Feature Magnitude learning (RTFM)  \cite{Tian_2021_ICCV} and the DRO based deep kernel MIL (DRO-DKMIL) \cite{Sapkota2021}. We also compare with a standard average top-$k$ model (Avg Topk) as the baseline. 
Avg Topk uses the rank loss in \eqref{eq:avg_topk_mil} with the same network architecture as BN-SVP. For RTFM, we get the result by re-running the original implementation for different $k$ values. Similarly, for DRO-DKMIL, we run the original implementation for different $\eta$ values that control the size of the uncertainty set.
The proposed BN-SVP removes the dependency on these highly sensitive parameters through non-parametric modeling. Detailed comparison results are shown in Figure \ref{fig: auc_k_eta}.

We have several important observations. First, all the top-$k$ models are very sensitive to the selection of the $k$ value (or $\eta$ that defines a soft version of the top-$k$ set). Both RTFM and DRO-DKMIL outperform the standard Avg Topk. DRO-DKMIL achieves relatively more stable performance across all datasets. This may attribute to its conversion of discrete optimization (\ie, choose a specific $k$) to a continuous optimization problem (\ie choosing $\eta$). However, for certain dataset (\eg, Avenue), its performance still varies more than $8\%$. Second, while for some rare cases that RTFM or DRO-DKMIL achieves the best performance for a specific $k$ or $\eta$, they under-perform BN-SVP in most cases. This is mainly due to that these models tend to choose a consecutive set of segments, which limits the model's exposure of other potentially positive segments. This issue has been effectively addressed by BN-SVP, which extracts a diverse set of potentially positive segments through submodular optimization. 


\vspace{-2mm}\paragraph{Comparison with Other Models}
We also make comparison with other existing techniques that do not depend on the $k$ value. Specifically, our comparison study includes the maximum score based MIL model (MMIL) by Sultani et al. \cite{Sultani2018}, attention based deep MIL model proposed by  Ilse et al. \cite{Ilse2018AttentionbasedDM}, a dictionary based approach proposed by Lu et al. \cite{Lu2013},  and an MIL model for soft bags (MILS) proposed by Li \& Vasconcelos \cite{Li2015} as common baselines for all datasets. Sultani et al. \cite{Sultani2018} used the loss function in \eqref{eq:max_mil} along with the temporal similarity and consistency as a regularizer. 
Ilse et al. used a permutation invariant aggregation function to detect the positive instances in the bag, where the function operators are learned using the attention network \cite{Ilse2018AttentionbasedDM}. Li \& Vasconcelos used a large-margin based latent support vector machine model with the goal to correctly classify positive and negative bags \cite{Li2015}. In case of the approach presented by Zhong et al. \cite{Zhong2019}, we directly report the performance from the original paper for the UCF-Crime and ShanghaiTech datasets. This approach involves multiple rounds of alternative optimization between classification and cleaning and may produce unstable
performance \cite{Tian_2021_ICCV}. Considering its difficulty in the training and replication process, we do not include it in other datasets.

Table~\ref{tab:auc_three_dataset} reports the AUC scores of BN-SVP along with the results from the comparison models as described above. It can be seen that BN-SVP clearly outperforms other models in all datasets and a large margin (\ie, $6$-$8\%$) is achieved on both ShanghaiTech and Avenue datasets. The corresponding ROC curves are shown in Figure~\ref{fig: roc_auc_sanghaitech_ucfcrime}, which demonstrates a consistent trend.  For example, on UCF-Crime, BN-SVP has a more than $10\%$ better True Positive Rate (TPR) compared to MMIL  at a False Positive Rate (FPR) of 0.2. Also at varying FPR, BN-SVP consistently outperforms the other competitive baselines, which justifies its outstanding detection capability.

\begin{table}[t!]
\caption{Comparison with Other Models}
\vspace{-6mm}
\begin{center}
\footnotesize
\begin{tabular}{|p{6cm}|c|}
\hline
Approach & AUC (\%) \\
\hline
\multicolumn{2}{|c|}{\sc UCF-Crime} \\
\hline
Hasan et al. \cite{Hasan2016} (C3D) & $50.60$ \\
\hline
Lu et al. \cite{Lu2013} (C3D) & $65.51$ \\
\hline
Lu et al. \cite{Lu2013} (I3D) & $61.98$ \\
\hline
MMIL \cite{Sultani2018} (C3D) & $75.41$ \\
\hline 
Li \& Vasconcelos \cite{Li2015} (I3D) & $77.95$ \\
\hline
Ilse et al. \cite{Ilse2018AttentionbasedDM} (I3D) & $76.52$ \\
\hline
Zhong et al. \cite{Zhong2019} (GCN (C3D)) & $81.08$ \\
\hline
Zhong et al. \cite{Zhong2019} ($\text{TSN}^\text{RGB}$) & $82.12$ \\
\hline
Zhong et al. \cite{Zhong2019} ($\text{TSN}^\text{OpticalFlow}$) & $78.08$ \\
\hline
MMIL  \cite{Sultani2018} (I3D) & $79.68$ \\
\hline
\textbf{BN-SVP (I3D)} & \textbf{83.39} \\
\hline
\multicolumn{2}{|c|}{\sc ShanghaiTech} \\
\hline
Lu et al. \cite{Lu2013} (C3D)  & $72.90$ \\
\hline
Li \& Vasconcelos \cite{Li2015} (C3D) & $90.40$ \\
\hline
Zhong et al. \cite{Zhong2019} (GCN (C3D)) & $76.44$ \\
\hline
Zhong et al. \cite{Zhong2019} ($\text{GCN} (\text{TSN}^\text{RGB})$) & $84.44$ \\
\hline
Zhong et al. \cite{Zhong2019} ($\text{GCN} (\text{TSN}^\text{Optical Flow})$) & $84.13$ \\
\hline
Ilse et al. \cite{Ilse2018AttentionbasedDM} (C3D) & $85.78$ \\
\hline
MMIL \cite{Sultani2018} (C3D)  & $92.18$ \\
\hline
\textbf{BN-SVP (C3D)} & $\textbf{96.00}$ \\
\hline
\multicolumn{2}{|c|}{\sc Avenue} \\
\hline
Binary SVM (C3D) & $69.11$\\
\hline
Lu et al. \cite{Lu2013} (C3D) & $62.14$ \\
\hline
Li \& Vasconcelos \cite{Li2015} (C3D) & $72.23$ \\
\hline
Ilse et al. \cite{Ilse2018AttentionbasedDM} (C3D) & $72.39$ \\
\hline
MMIL \cite{Sultani2018} & $70.40$  \\
\hline
\textbf{BN-SVP (C3D)}  & $\textbf{80.87}$ \\
\hline
\end{tabular}
\label{tab:auc_three_dataset}
\end{center}
\vspace{-7mm}
\end{table}

\begin{figure*}[t!]
\vspace{-3mm}
\centering
\begin{subfigure}{0.19\textwidth}
  \centering
  \includegraphics[angle=90,width=\linewidth]{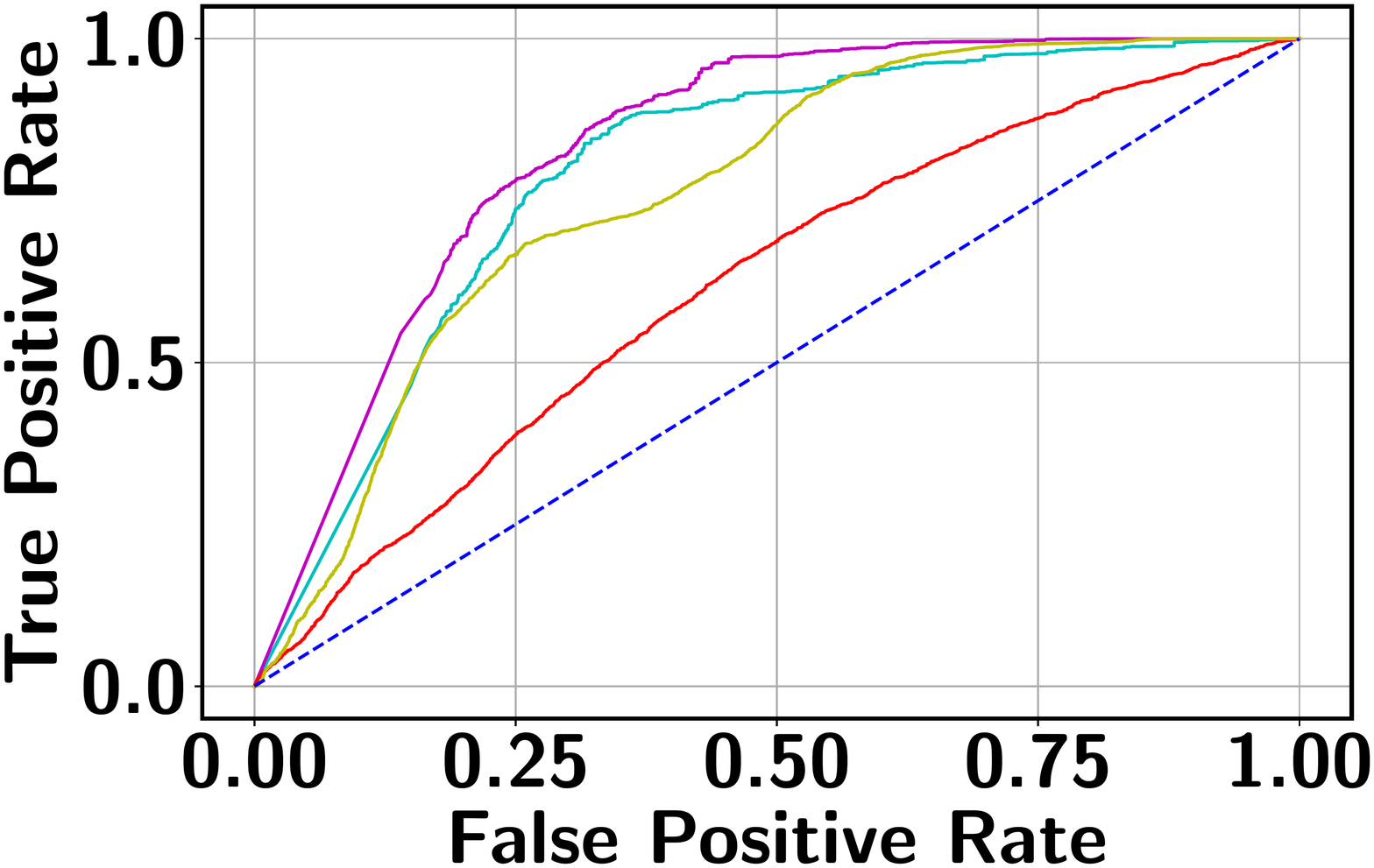}
  \vspace{-4mm}
  \caption{UCF-Crime}
\end{subfigure}
\begin{subfigure}{.19\textwidth}
  \centering
  \includegraphics[angle=90,width=\linewidth]{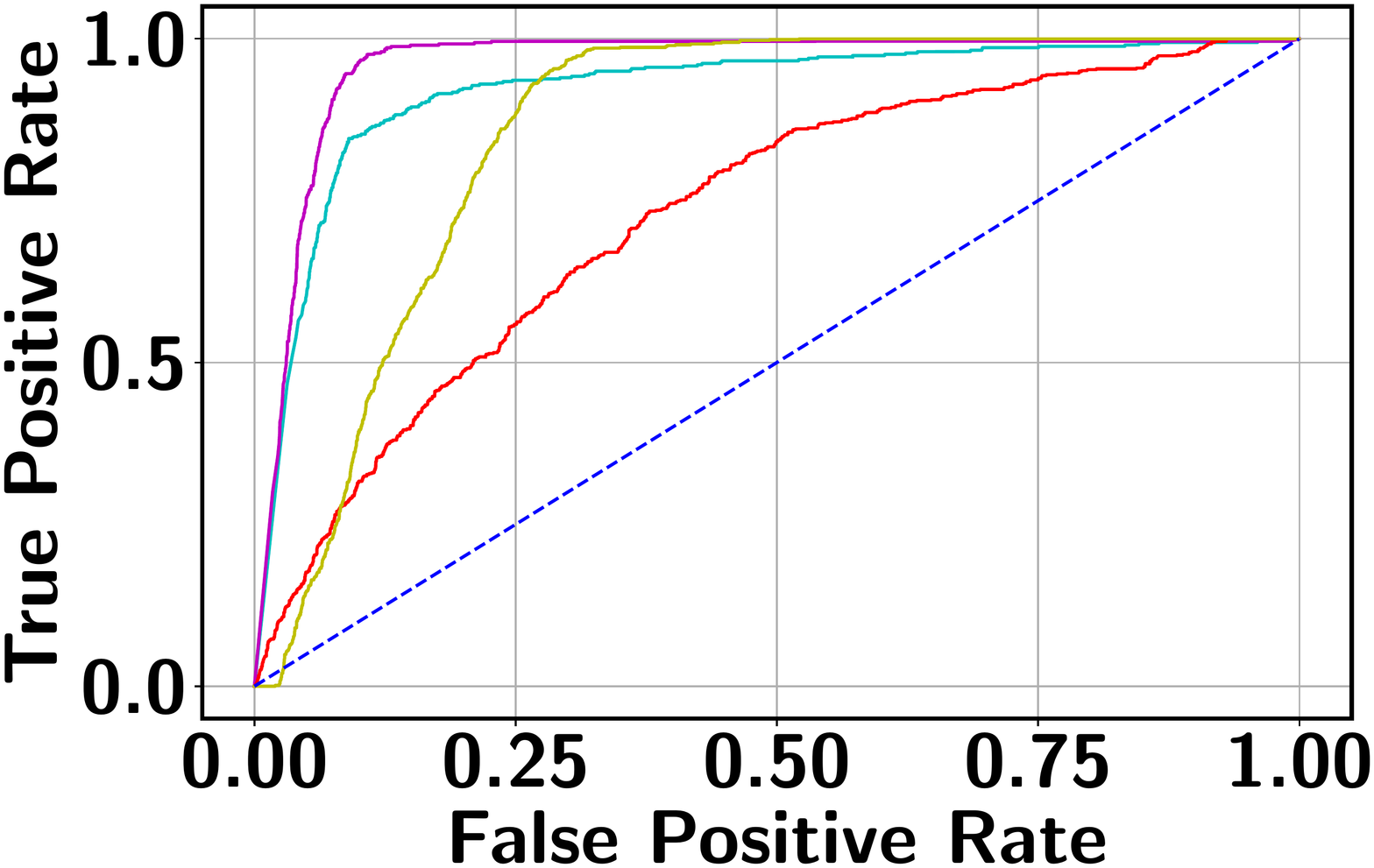}
  \vspace{-4mm}
  \caption{ShanghaiTech}
\end{subfigure}%
\begin{subfigure}{0.19\textwidth}
  \centering
  \includegraphics[angle=90,width=\linewidth]{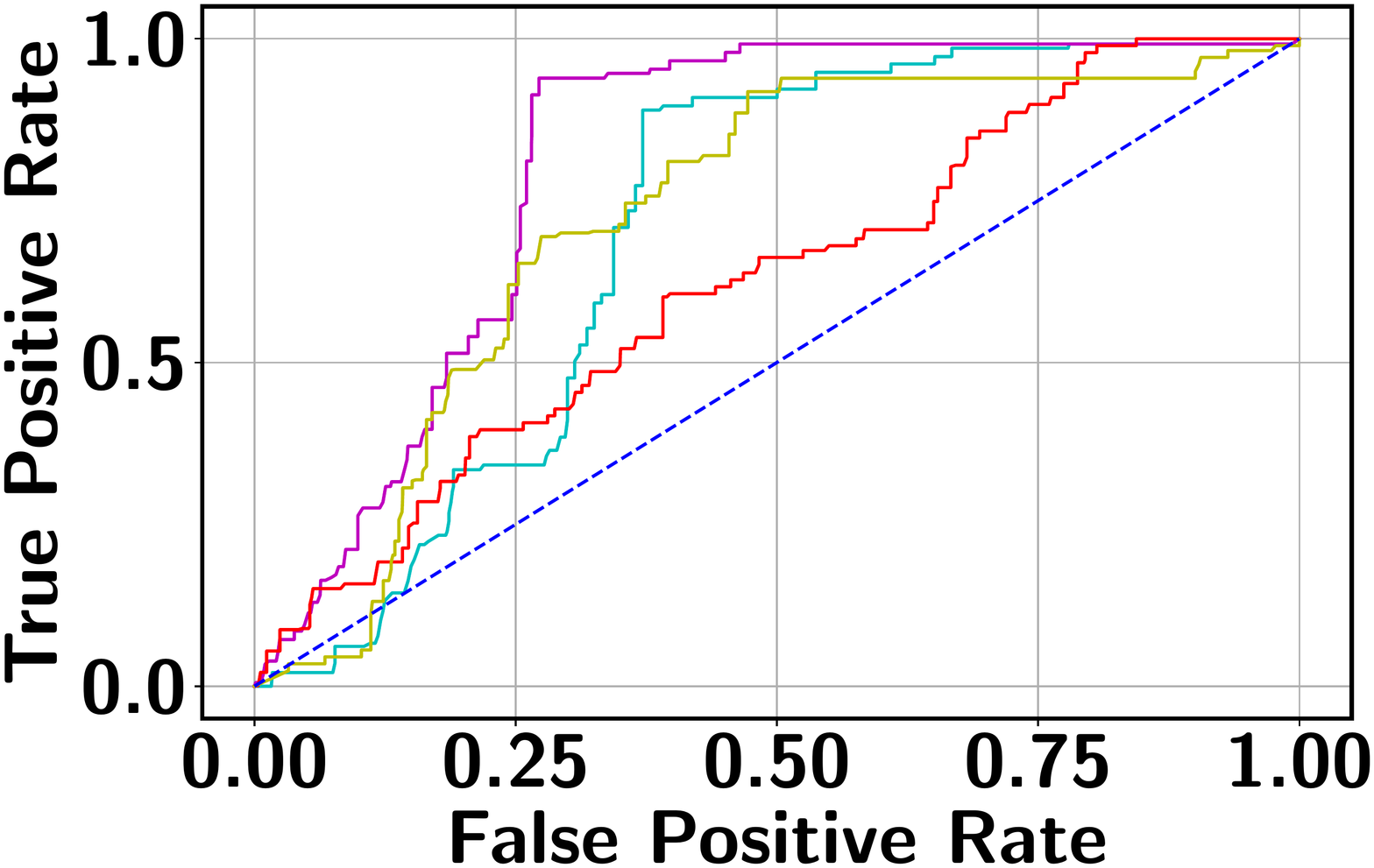}
  \vspace{-4mm}
  \caption{Avenue}
\end{subfigure}
\begin{subfigure}{0.19\textwidth}
  \centering
  \includegraphics[angle=90,width=\linewidth]{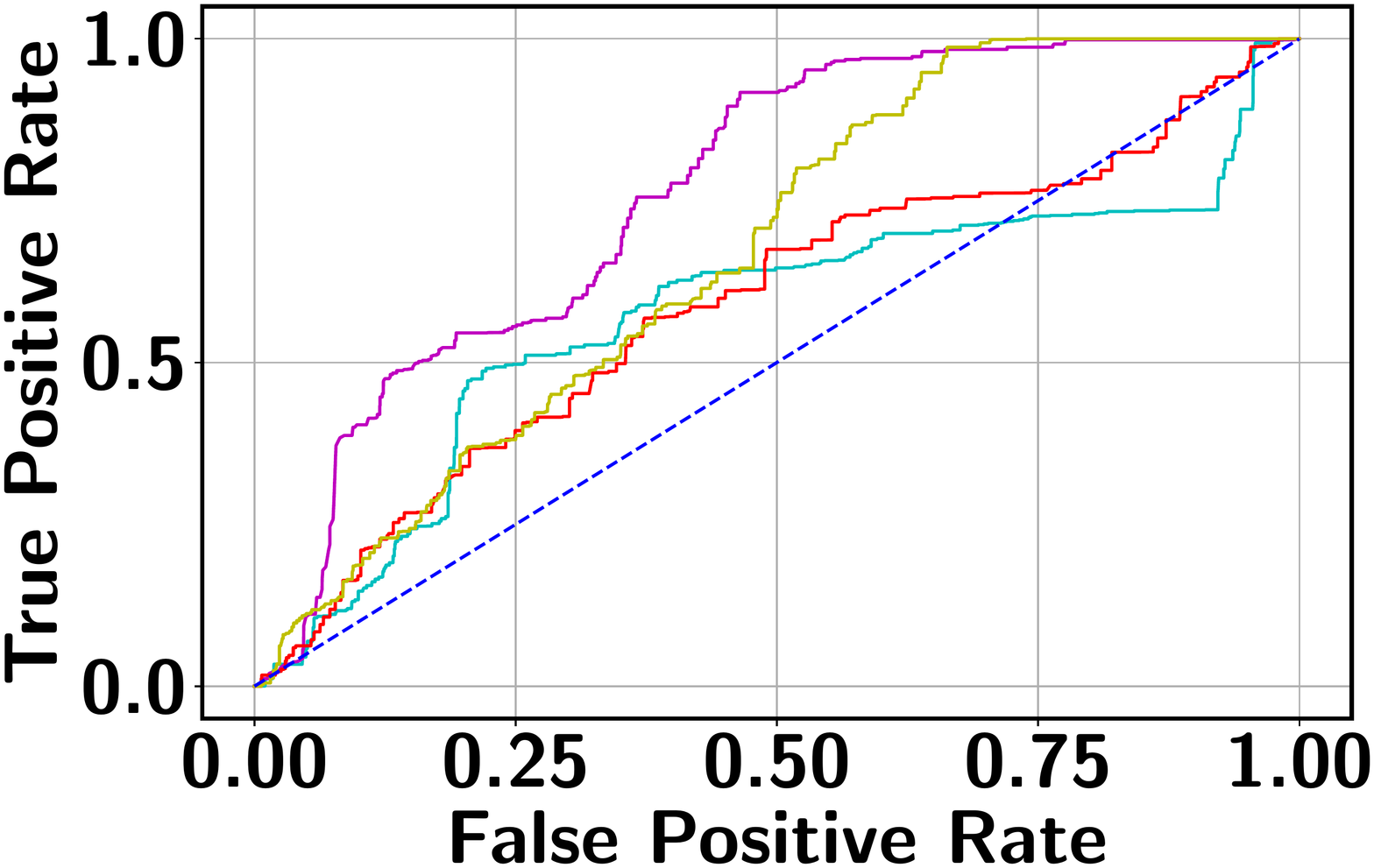}
  \vspace{-4mm}
  \caption{Multimodal (UCF)}
\end{subfigure}
\begin{subfigure}{0.19\textwidth}
  \centering
  \includegraphics[angle=90,width=\linewidth]{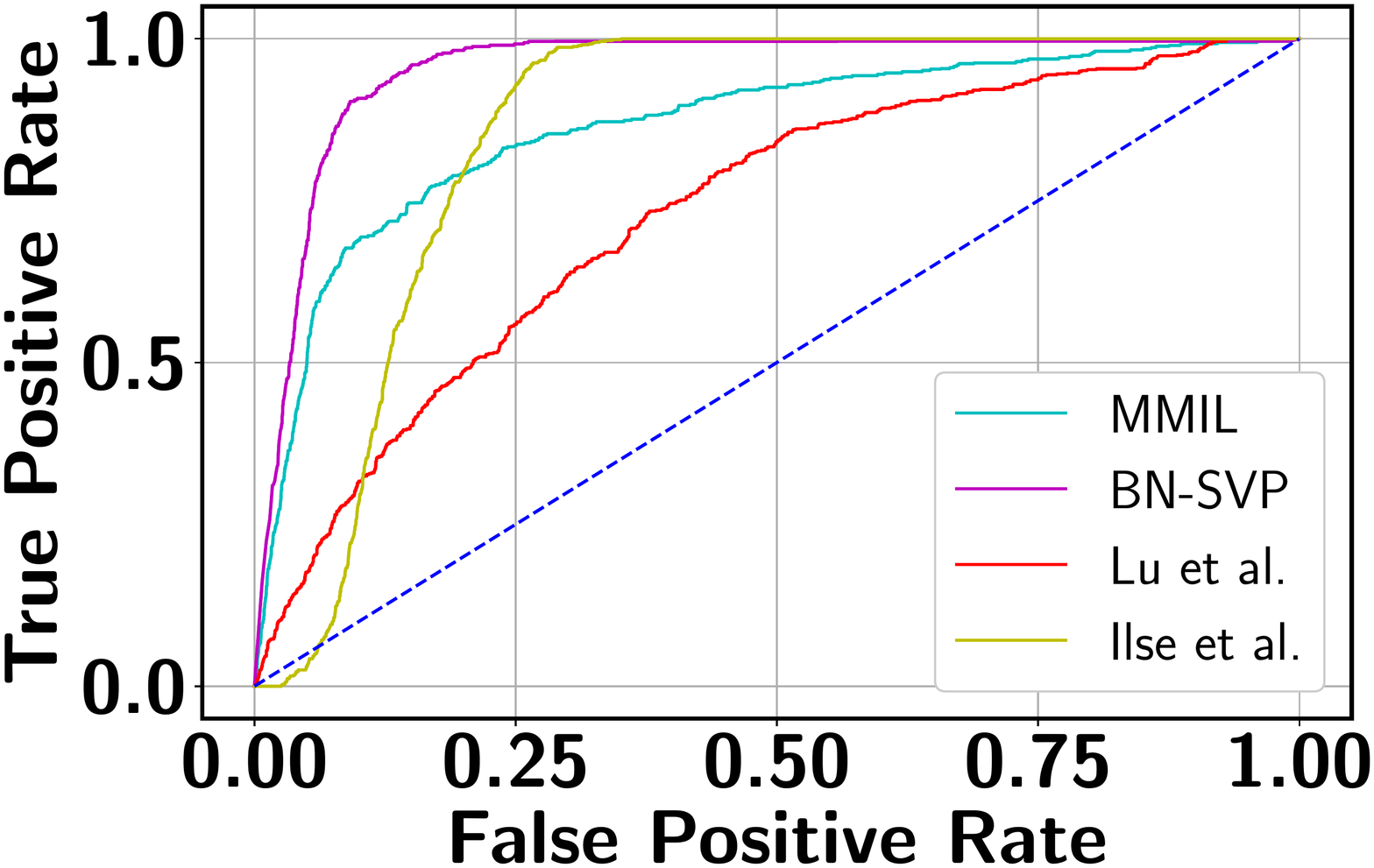}
  \vspace{-4mm}
  \caption{Outlier (ShanghaiTech)}
\end{subfigure}
\vspace{-2mm}
\caption{ROC curves on three video datasets (a)-(c), multimodal (d) and outlier (e)} 
\vspace{-4mm}
\label{fig: roc_auc_sanghaitech_ucfcrime}
\end{figure*}



\subsection{Detecting Multimodal and Outlier Segments}\label{sec:multimodal}

\vspace{-3mm}\paragraph{Multimodal Detection} The original UCF-Crime dataset does not explicitly consider a multimodal scenario. Even though the real-world surveillance videos may indeed contain those cases (which is evidenced by the superior performance of the BN-SVP model), it is hard to identify actual videos for this specific information. In UCF-Crime dataset, different types of anomalies are present. This allows us to explicitly create multimodal scenarios by combining multiple abnormal videos from different activity types. To this end, we randomly select three activity types and form an abnormal bag by concatenating three abnormal videos, one video per activity type. The training bags are constructed using the training dataset whereas testing bags are constructed using the testing dataset. In total, we construct 50 abnormal and 50 normal training bags. In the testing set, there are 10 normal and 10 abnormal videos. Table~\ref{tab: result_multimodal_outlier} shows the AUC scores and corresponding ROC curve is shown in the Figure~\ref{fig: roc_auc_sanghaitech_ucfcrime} (d). BN-SVP achieves a more superior performance compared to other baselines. Furthermore, BN-SVP stays consistently on the top in the ROC curve justifying the effectiveness of the approach toward the multimodal scenario. As an example, at $\text{FPR} = 0.1$, BN-SVP is at least $20\%$ better  than other approaches on TPR. 

\vspace{-2mm}\paragraph{Outlier Detection} To assess the robustness on outlier detection, we extend the ShanghaiTech dataset with outliers. Specifically, we randomly select 120 segments from abnormal videos and replace their features with points drawn from a standard multivariate Gaussian distribution. As shown in Table~\ref{tab: result_multimodal_outlier}, MMIL suffers heavily by the outliers compared to the proposed BN-SVP. This is because, it is likely to have an outlier prediction as the maximum prediction score from the abnormal video. As a result, the overall optimization process may be heavily influenced by outliers.

\begin{table}[t!]
   
\small
\begin{center}
 \caption{AUC Scores on Multimodal and Outlier Detection}
   \vspace{-3mm}
\label{tab: result_multimodal_outlier}
\begin{tabular}{|p{4.0cm}|c|c|c|}
\hline
\textbf{Approach}&\multicolumn{2}{|c|}{\textbf{AUC ($\%$)}} \\
\cline{2-3} 
\textbf{} & \textbf{\textit{Multimodal}}& \textbf{\textit{Outlier}} \\
\hline
Lu et al. \cite{Lu2013} (C3D) & $58.67$ & $72.90$  \\
\hline
Li \& Vasconcelos \cite{Li2015} (C3D) & $70.96$ & $90.95$  \\
\hline
Ilse et al. \cite{Ilse2018AttentionbasedDM} (C3D) & $66.85$ & $85.65$  \\
\hline
MMIL \cite{Sultani2018}  & $57.08$ & $86.47$ \\
\hline
\textbf{BN-SVP} & $\textbf{76.53}$ & $\textbf{95.27}$ \\
\hline
\end{tabular}
\end{center}
\vspace{-7mm}
\end{table}

\subsection{Qualitative Analysis}
\vspace{-2mm}
\begin{figure}[t!]
\centering
\begin{subfigure}{0.22\textwidth}
  \centering
  \includegraphics[width=.8\linewidth]{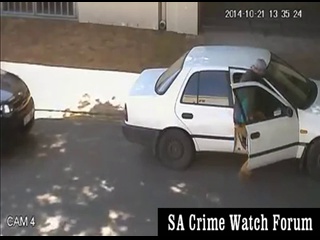}
  \caption{Frame 1}
\end{subfigure}
\begin{subfigure}{0.22\textwidth}
  \centering
  \includegraphics[width=.8\linewidth]{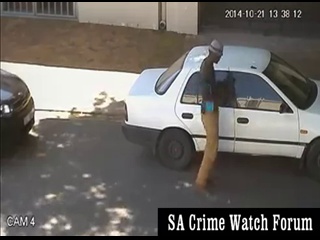}
  \caption{Frame 2}
\end{subfigure}
\vspace{-2mm}
\caption{Frames from UCF-Crime Stealing019; (a) Correct BN-SVP, Avg Topk, (b) Correct BN-SVP, Incorrect Avg Topk  }
\label{fig: qualitative_ucfcrime}
\vspace{-7mm}
\end{figure}
 
To show the effectiveness of extracting a diverse set of segments for model training, we present illustrative sample frames in a stealing video from UCF-Crime, where BN-SVP correctly identifies all abnormal frames and a top-$k$ approach (\eg, Avg Topk) misses some of them. In Figure~\ref{fig: qualitative_ucfcrime}, both frames are of abnormal types and but they occur in two distinct time intervals within the video. The first frame is more obvious for a stealing event. Consequently,  both the proposed BN-SVP and Avg Topk are able to correctly identify it. In contrast, the second frame is less obvious for a stealing activity. Nevertheless, it is still chosen by BN-SVP due to its diverse coverage of potentially abnormal frames during the training process. 
On the other hand, Avg Topk only focuses on frames with high prediction scores that are usually co-located in the same time interval. This will narrow the scope of the model being exposed to other abnormal frames.  
Therefore, Avg Topk is not able to correctly predict the second frame and falsely classify it as normal. More qualitative analysis that demonstrates the robustness of the proposed approach on multimodal and outlier scenarios is provided in the Appendix along with an ablation study for the prediction score threshold $\epsilon$ defined in \eqref{eq:score-selection}.
\vspace{-2mm}
\section{Conclusion}
\vspace{-2mm}
In this paper, we propose a novel Bayesian non-parametric submodularity diversified MIL model for robust video anomaly detection in practical settings that involve outlier and multimodal scenarios. By integrating submodular optimization with the minimization of an MIL loss, the proposed approach identifies a diverse set of segments to ensure comprehensive coverage of all potential positive segments for effective model training. The Bayesian non-parametric construction of the submodular set function automatically determines the upper bound on the size of the diverse set, which serves as a key constraint for minimizing the submodularity diversified MIL loss function. The resulting state-component structure also leads to a greedy submodular optimization algorithm to support efficient model training.  The effectiveness of the proposed approach is demonstrated through the state-of-the-art robust anomaly detection performance on real-world surveillance videos with noisy and multimodal scenarios. 
\vspace{-2mm}
\section*{Acknowledgement}
\vspace{-2mm}
This research was supported in part by an NSF IIS award IIS-1814450 and an ONR award N00014-18-1-2875. The views and conclusions contained in this paper are those of the authors and should not be interpreted as representing any funding agency.


\newpage
\bibliographystyle{ieee_fullname}
\bibliography{macros,main}

\newpage
\appendix
\onecolumn

\begin{center}
    {\large \bf Appendix}
\end{center}

\paragraph{Organization of Appendix}
In this appendix,  we first summarize the major notations used in the paper in Table~\ref{tab: symbol_table} of Appendix~\ref{app:notations}. We then provide the detailed proof of Theorem 1 in Appendix~\ref{sec:proofs}. Next, we provide additional experimental results in Appendix~\ref{sec:additional_experiments}. Finally, the link to the source code is provided in Appendix~\ref{sec:link_source_code}.

\section{Summary of Notations}\label{app:notations}

Table~\ref{tab: symbol_table} summarizes all the major symbols along with their descriptions.
\begin{table*}[htbp]
\caption{Symbols with Descriptions}
\vspace{-3mm}
\small
\begin{center}
\begin{tabular}{|c|c|}
\hline
Notation & Description \\
\hline
$\mathcal{B}_{pos}$ & Positive bag (video) \\
\hline
$\mathcal{B}_{neg}$ & Negative bag (video) \\
\hline
$n$ & Number of segments in each bag \\
\hline
$\textbf{x}_i^+$ & Segment in a positive bag \\
\hline
${\bf x}_{[i]}^+$ & $i^{th}$ largest prediction segment in a positive bag \\
\hline
$\textbf{x}_j^-$ & Segment in a negative bag \\
\hline
$M$ & Feature dimension of each video segment \\
\hline
\textbf{w} & Network parameters \\
\hline
b & Network bias \\
\hline
$k$ & Number of segments considered in the top-$k$ formulation \\
\hline

$\eta$ & Learning rate \\
\hline
%
$\mathcal{C}^+$ & Set of instances from positive bag involve in model training \\
\hline
$\mathcal{G}_0$ & Base distribution in DP \\
\hline
$\gamma$ & Concentration parameter for the distribution $\mathcal{G}_0$ \\
\hline
$\beta_k$ & Weight associated with the $k^{th}$ atom \\
\hline
$\phi_k$ & Atom $k$ drawn from the distribution $H$ \\
\hline
$\mathcal{G}_j$ & Transition probability distribution of $j^{th}$ state \\
\hline
$\hat{\pi}_{jl}$ & Stick breaking weight associated with $l^{th}$ atom in $j^{th}$ group \\
\hline
$\alpha$ & Concentration parameter for $\hat{\pi}_{j}$ \\
\hline
$\phi_{jl}$ & $l^{th}$ atom corresponding $j^{th}$ group \\
\hline
$\beta_k$ & Stick breaking weight corresponding to atom $\phi_k$ \\
\hline
$\gamma$ & Concentration parameter for $\beta_k$ \\
\hline
$\rho$ & Parameter defining the self transitioning \\
\hline
$z_i$ & Scene assignment for the $i^{th}$ segment in a video \\
\hline
$s_i$ & Mixture component assignment for the $i^{th}$ segment in a video \\
\hline
$\mathcal{N}$ & Multivariate Gaussian distribution \\
\hline
$\boldsymbol{\mu}_{k, t}$ & Mean of the $k^{th}$ state, $t^{th}$ mixture component \\
\hline
$\Sigma_{k, t}$ & Covariance of $k^{th}$ state, $t^{th}$ mixture component \\
\hline
$S_{i, j}$ & Pairwise similarity between $i^{th}$ and $j^{th}$ segments \\
\hline
$F(\mathcal{C})$ & Submodular set function \\
\hline
$f_s^*$ & Maximum output score among segments assigned to the same cluster \\
\hline
$i_s^*$ & Index of the representative segment \\
\hline
$\widehat{{\mathcal C}^+}$ & Representative set constructed using the greedy algorithm \\
\hline
$\epsilon$ & Threshold to exclude segments with low prediction score from the representative set \\
\hline
$\kappa$ & Upper bound of number of representative segments \\

\hline

\end{tabular}
\label{tab: symbol_table}
\end{center}
\vspace{-4mm}
\end{table*}

\section{Proof of Theorem 1}
\label{sec:proofs}
In this section, we provide the detailed proof of Theorem 1. We first show that the representative set based MIL loss given by \eqref{eq:representative_topk_mil} is equivalent to the submodularity diversified MIL loss given by Equation \eqref{eq: sub_cons_avg_topk} with a specific $\lambda$ to balance the MIL loss and the diversity of the set. We then show that greedy algorithm to locate the $\kappa$ representative segments provides a $\kappa$-constrained greedy approximation to the maximization of the submodular set function $F(C)$ with the solution to be no worse than $(1-e^{-1})$ of the optimal solution.

\paragraph{Proof of representative set based MIL loss in \eqref{eq:representative_topk_mil} is a special case of the submodular diversified MIL loss in \eqref{eq: sub_cons_avg_topk}} We first present a lemma, which is used in the proof. 

\begin{lemma}
Assume that $\widetilde{\mathcal{C}^+}$ with size $\kappa$ is a solution that maximizes $F(\mathcal{C})$ in \eqref{eq: submodular_function}. Then, $\widetilde{\mathcal{C}^+}$ should contain one segment from each mixture component (\ie sub-scene).
\end{lemma}
\begin{proof}
The lemma can be proved by following  the definition of the BN-SVP induced pairwise similarity between segments given by \eqref{eq:similarity} and then use proof by contradiction. Assume that at least two segments, say ${\bf x}^{(t)}_i, {\bf x}^{(t)}_j$, are chosen from the same component $t$. Then, there will be at least one component, say $t'$, where no segments are chosen by $\widetilde{\mathcal{C}^+}$. Given the definition of $F(\mathcal{C})$ in \eqref{eq: submodular_function}, for each segment in $t$, either ${\bf x}^{(t)}_i$ or ${\bf x}^{(t)}_j$ could be used to compute the pairwise similarity based on their closeness to that segment. Since the cohesiveness of each component is guaranteed through the BN-SVP process, both ${\bf x}^{(t)}_i$ and ${\bf x}^{(t)}_j$ should be close to the mean of their assigned Gaussian component $\mathcal{N}({\bf x}_t, \Sigma_t)$ to ensure a high likelihood optimized by HDP-HMM. Due to triangle inequality, ${\bf x}^{(t)}_i$ and ${\bf x}^{(t)}_j$ should be close to each other. As a result, we can assume that ${\bf x}^{(t)}_i$ is always chosen to evaluate the pairwise similarity $S_{i,p}$ with each segment ${\bf x}^{(t)}_p$ in component $t$. Next, we replace ${\bf x}^{(t)}_j$ with another segment  ${\bf x}^{(t')}_j$ from component $t'$ to construct another solution set $\overline{\mathcal{C}^+}$. Since ${\bf x}^{(t')}_j$ has positive similarity with each segment in $t'$ and the pairwise similarity between ${\bf x}^{(t)}_j$ and all segments in $t'$ is all zero, we have $F(\overline{\mathcal{C}^+})>F(\widetilde{\mathcal{C}^+})$, which contradicts the assumption that $\widetilde{\mathcal{C}^+}$ maximizes $F(\mathcal{C})$. 
\end{proof}
Since the representative set $\widehat{\mathcal{C}^+}$ is constructed by choosing one segment from each mixture component, it satisfies the necessary condition to be an optimizer of $F(\mathcal{C})$ specified in the above lemma. However, choosing a set of segments with the maximum diversity is not the primary goal and the overall objective function \eqref{eq: sub_cons_avg_topk} includes both the MIL loss and the diversity, which are balanced through $\lambda$. Due to the lack of instance-level labels, choosing a $\lambda$ that optimally balances the MIL loss and the set diversity is challenging. We argue that construction $\widehat{\mathcal{C}^+}$ essentially offers alternative way to set a specific $\lambda$ to balance these two terms. First, since the constraint $|\mathcal{C}^+|\leq {\kappa}$ allows the set to contain less than $\kappa$ segment, $\widehat{\mathcal{C}^+}$ excludes those segments with low prediction scores. This can be viewed as setting a $\lambda$ to increase $-F(\mathcal{C}^+)$ while decreasing the MIL loss $L(\mathcal{C}^+)$. Similarly, instead of choosing the instance with the largest pairwise similarity with all other segments in the same component, we choose a segment with the highest prediction score. Again, this can be viewed as further reducing the $\lambda$ to give more preference to the MIL loss as such segments can further reduce the training MIL loss. Thus, instead of directly setting $\lambda$, which is highly challenging, $\widehat{\mathcal{C}^+}$ is constructed by leveraging both the mixture component assignments and the prediction scores of the segments. This is equivalent to implicitly setting a $\lambda$ to balance the MIL loss and the diversity of the representative set $\widehat{\mathcal{C}^+}$, which completes the proof of the equivalence of these two objective functions.

\paragraph{Proof of the optimality of the greedy algorithm}
We first reformulate \eqref{eq: sub_cons_avg_topk} as a minimization problem $\min_{{\bf w}}g({\bf w})$ with $g({\bf w})$ defined as
\begin{equation}
    g({\bf w}) \delequal \min_{{\mathcal C}^+\subseteq \mathcal{B}_{pos},  |{\mathcal C}^+|\leq \kappa} L(\mathcal{B}_{pos}, \mathcal{B}_{neg})-\lambda F(\mathcal{C}^+)
\end{equation}
The above optimization involves finding a subset ${\mathcal C}^+\subseteq \mathcal{B}_{pos}$ that maximizes $F({\mathcal C}^+)$. This requires enumerating over all ${n\choose \kappa}$ possible subsets, which is expensive when there are large number of segments in a given video. Defining the discrete objective function $G_{{\bf w}}$ where 
\begin{equation}
    G_{{\bf w}}({\mathcal C}^+) \delequal L(\mathcal{B}_{pos}, \mathcal{B}_{neg})-\lambda F(\mathcal{C}^+)
\end{equation}
Since $-G_{{\bf w}}({\mathcal C}^+)$ is monotone non-decreasing submodular, a fast greedy procedure can be used to approximately optimize $G_{{\bf w}}({\mathcal C}^+)$. A typical greedy procedure involves evaluating the similarity between each pair of segments in a video and then choose the segments with the largest overall similarity with the all other segments. We make two important adjustments of this standard greedy process. First, our non-parametric HDP-HMM process follows the clustering based heuristic (Lin and Bilmes 2018) by choosing one segment from each cluster, which avoids evaluating each candidate segment in the video. Different from (Lin and Bilmes 2018), which chooses the data point that is closest to the cluster centroid, we choose the one with the highest output score. Second, our similarity evaluation takes a linear complexity with respect to the bag size by leveraging the temporal neighborhood of the segments. 
By leveraging the above greedy procedure, we can show that the obtained approximate  solution is guaranteed to be no worse than $(1-e^{-1})$ of the optimal solution according to the standard result from
(Nemhauser et al. 1978), which completes the proof of the second part. 


\begin{table*}[h!]
 \caption{Video Level Distribution on Different Datasets}
\label{tab: video_level_distribution}  
\vspace{-4mm}
\small
\begin{center}
\begin{tabular}{|c|c|c|c|c|c|c|c|c|}
\hline
\textbf{Split}&\multicolumn{2}{|c|}{\textbf{ShanghaiTech}} &\multicolumn{2}{|c|}{\textbf{UCF-Crime}} &\multicolumn{2}{|c|}{\textbf{UCF-Crime Multimodal}} &\multicolumn{2}{|c|}{\textbf{Avenue}}\\
\cline{2-9} 
\textbf{} & \textbf{\textit{Normal}}& \textbf{\textit{Abnormal}}& \textbf{\textit{Normal}} & \textbf{\textit{Abormal}} & \textbf{\textit{Normal}} & \textbf{\textit{Abnormal}} & \textbf{\textit{Normal}} & \textbf{\textit{Abnormal}}\\
\hline
Train & $175$ & $63$ & $810$ & $800$ & $150$ & $150$ & $13$ & $17$  \\
\hline
Test & $155$ & $44$ & $150$ & $140$ & $30$ & $30$ & $3$ & $4$ \\
\hline
\end{tabular}
\end{center}

\vspace{-6mm}
\end{table*}
\section{Additional Experimental Results}
\label{sec:additional_experiments}
In this section, we first show the detailed network architecture used in our training process. Next, we provide the ablation study demonstrating the impact of hyperparameter $\epsilon$ used in our experimentation. Finally, we provide some additional qualitative analysis justifying the effectiveness of the proposed approach. Further we also show effectiveness of the HDP-HMM technique to discover subscenes of different types in a video through qualitative analysis.
\subsection{Network Architecture}
First, we pass the pre-trained features through the two parallel GCN branches. The upper branch captures the feature similarity between segments and the lower one captures the temporal consistency between segments such that nearby segments will provide similar predictions. The output of the parallel branches are combined and passed through the 5 LSTM layers with 32 hidden units on each. Next, the output is passed through the BatchNorm. Finally, FCN is connected with sigmoid activation to get the final prediction score.

\paragraph{GCN Architecure}
Next, we explain the GCN architecture in detail. Let $\mathbf{A}$ is the $n\times n$ dimensional adjacency matrix where the $(i, j)$ entry in the matrix indicates the similarity between segment $i$ and $j$. Mathematically,
\begin{equation}
    \mathbf{A}(i, j) = k({\bf x}_i, {\bf x}_j) 
\end{equation}
where ${\bf x}_i$ and ${\bf x}_j$ be the D-dimensional representation for $i^{th}$ and $j^{th}$ segments respectively. It should be noted that for the feature similarity branch, we use the RBF kernel with the following form
\begin{equation}
    k({\bf x}_i, {\bf x}_j) = \exp(\frac{|{\bf x}_i-{\bf x}_j|^2}{-2l^2})
\end{equation}
In case of temporal consistency branch, we use the following form between $i^{th}$ and $j^{th}$ segment:
\begin{equation}
            k({\bf x}_i, {\bf x}_j) = \exp(-|i-j|)
        \end{equation}
This drives the temporally nearby segments to have a similar score. Based on the adjacency matrix, following Kipf and Welling \cite{Kipf:2016tc}, the graph-Laplacian with the renormalization trick can be written as
    
    \begin{equation}
        \hat{{\bf A}} = {\bf D}^{-\frac{1}{2}}({\bf A}+{\bf I_n}){\bf D}^{-\frac{1}{2}}
    \end{equation}
In the above equation ${\bf D}_{(i, i)} = \sum_j\left\{{\bf A}+{\bf I_n}\right\}_{(i, j)}$ is the corresponding degree matrix.
The output of the feature similarity graph is computed as:
    \begin{equation}
        {\bf H} = {\bf \hat{A}}{\bf X}{\bf W}
    \end{equation}
where ${\bf W} \in \mathcal{R}^{D\times M}$ is a trainable parameter matrix and ${\bf X}\in \mathcal{R}^{n\times D}$ is the video specific features.  

\subsection{Ablation Study}
\paragraph{Impact of $\epsilon$}
In this subsection, we show the the impact of the error threshold $\epsilon$ on the model performance. It is worth mentioning that $\epsilon$ indicates the percentile we used to determine the threshold so as to exclude the clusters with potentially all normal segments with a high probability. For example, $\epsilon=0.1$ indicates that we first determine the output score  corresponding to the segment that lies in the $10^{th}$ percentile based on scores of all segments sorted in the non-decreasing order. This selected score is used as the threshold. Next, all representative segments with a predicted score below this output threshold are discarded from the representative set $\widehat{\mathcal{C}^+}$.  Figure~\ref{fig: ablation_epsilon} show the performance variation with respect to the different $\epsilon$'s for five different datasets. As can be seen, for a relatively lower $\epsilon$ value (\ie, 20-35\%), the performance is fairly stable for all datasets. This is because, with a low $\epsilon$, the model rejects the segments from a given video with a sufficiently low output score. This way, the chance of including normal segments from abnormal videos is minimized. Further lowering $\epsilon$ may include a good number of normal segments, making the model mis-identify other similar normal segments as anomalies. On the other hand, choosing a very high $\epsilon$ results in the drop in performance. In this case, some potentially abnormal segments may be missed in the loss function and therefore, the model may have less exposure to different types of abnormal frames resulting in the degradation of performance. In sum, as long as we stay in the relatively low range when choosing $\epsilon$ (\eg, 20-35\% gives very stable results), we can get the stable (and nearly optimal) performance.

\begin{figure*}[t!]
\vspace{-2mm}
\centering
\begin{subfigure}{0.19\textwidth}
  \centering
  \includegraphics[angle=90,width=\linewidth]{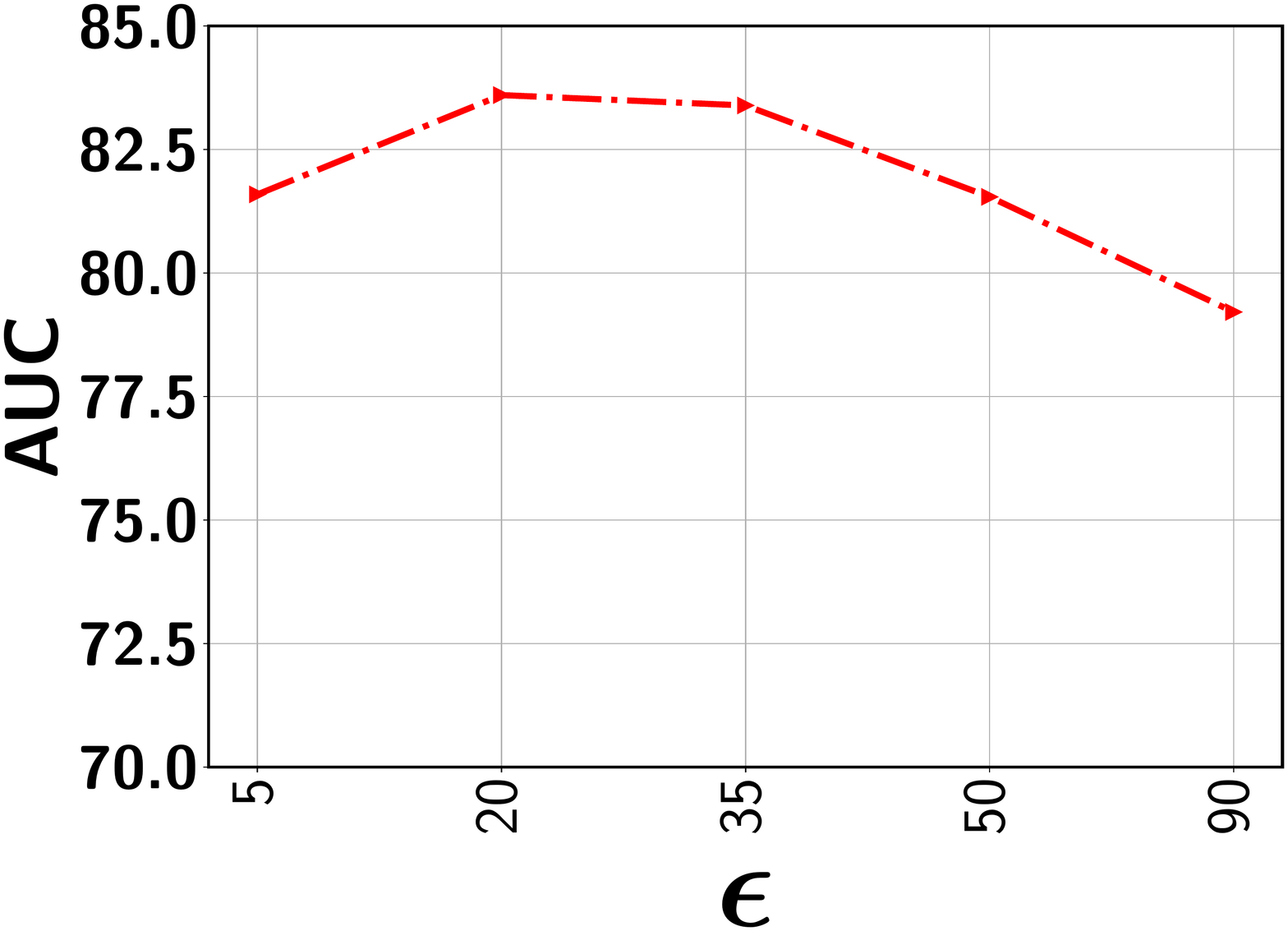}
  \vspace{-4mm}
  \caption{UCF-Crime}
\end{subfigure}
\begin{subfigure}{.19\textwidth}
  \centering
  \includegraphics[angle=90,width=\linewidth]{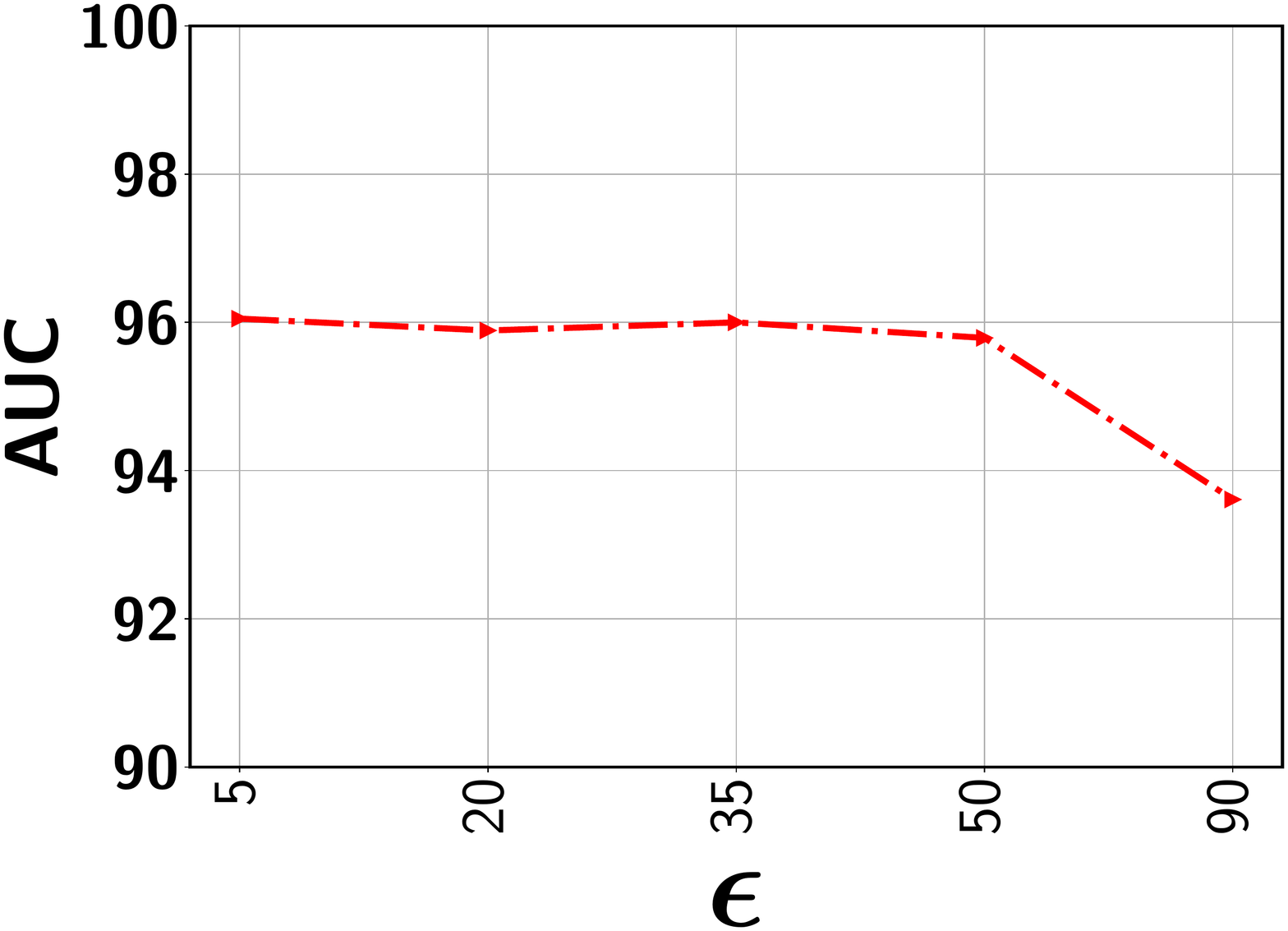}
  \vspace{-4mm}
  \caption{ShanghaiTech}
\end{subfigure}%
\begin{subfigure}{0.19\textwidth}
  \centering
  \includegraphics[angle=90,width=\linewidth]{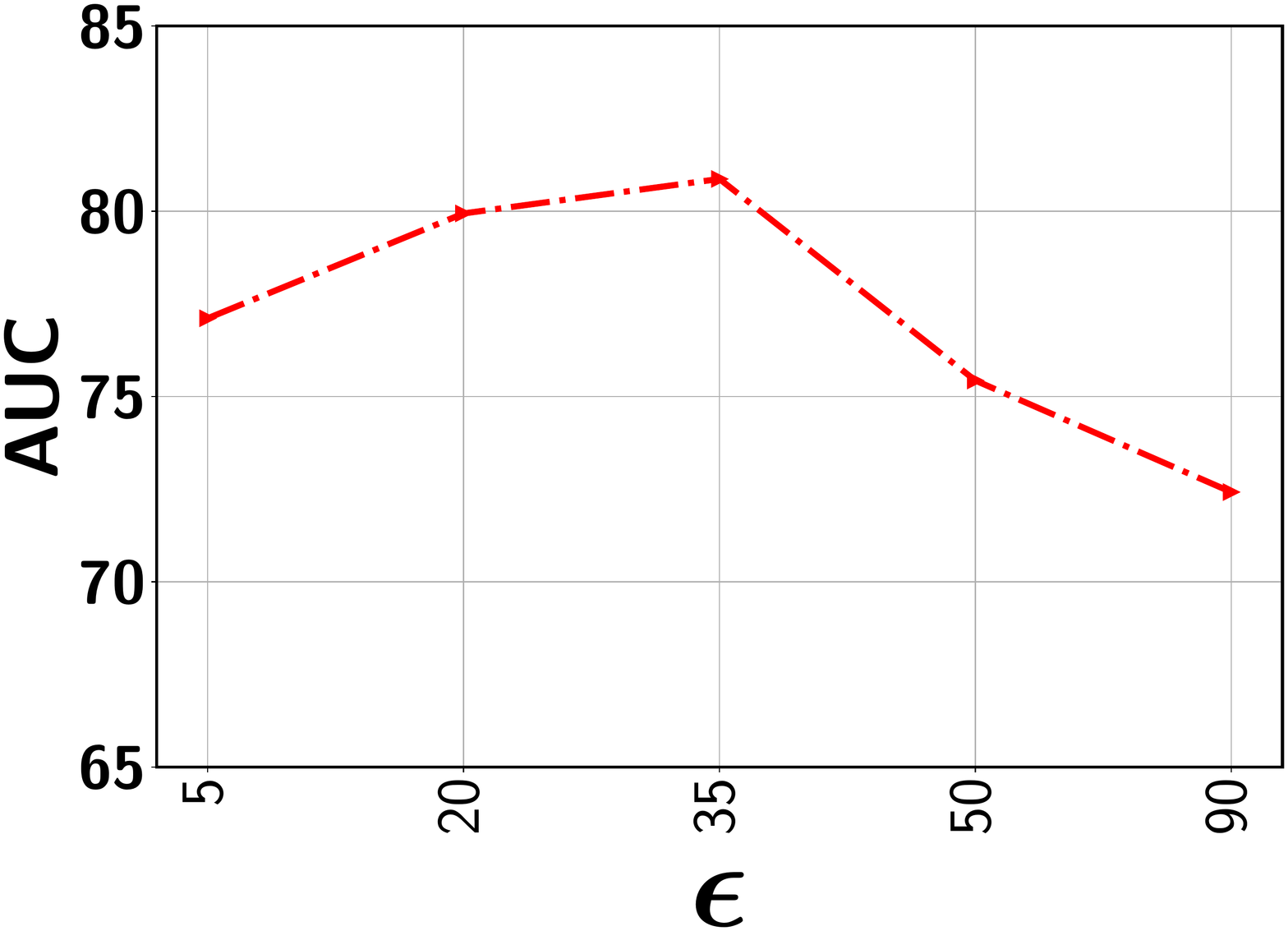}
  \vspace{-4mm}
  \caption{Avenue}
\end{subfigure}
\begin{subfigure}{0.19\textwidth}
  \centering
  \includegraphics[angle=90,width=\linewidth]{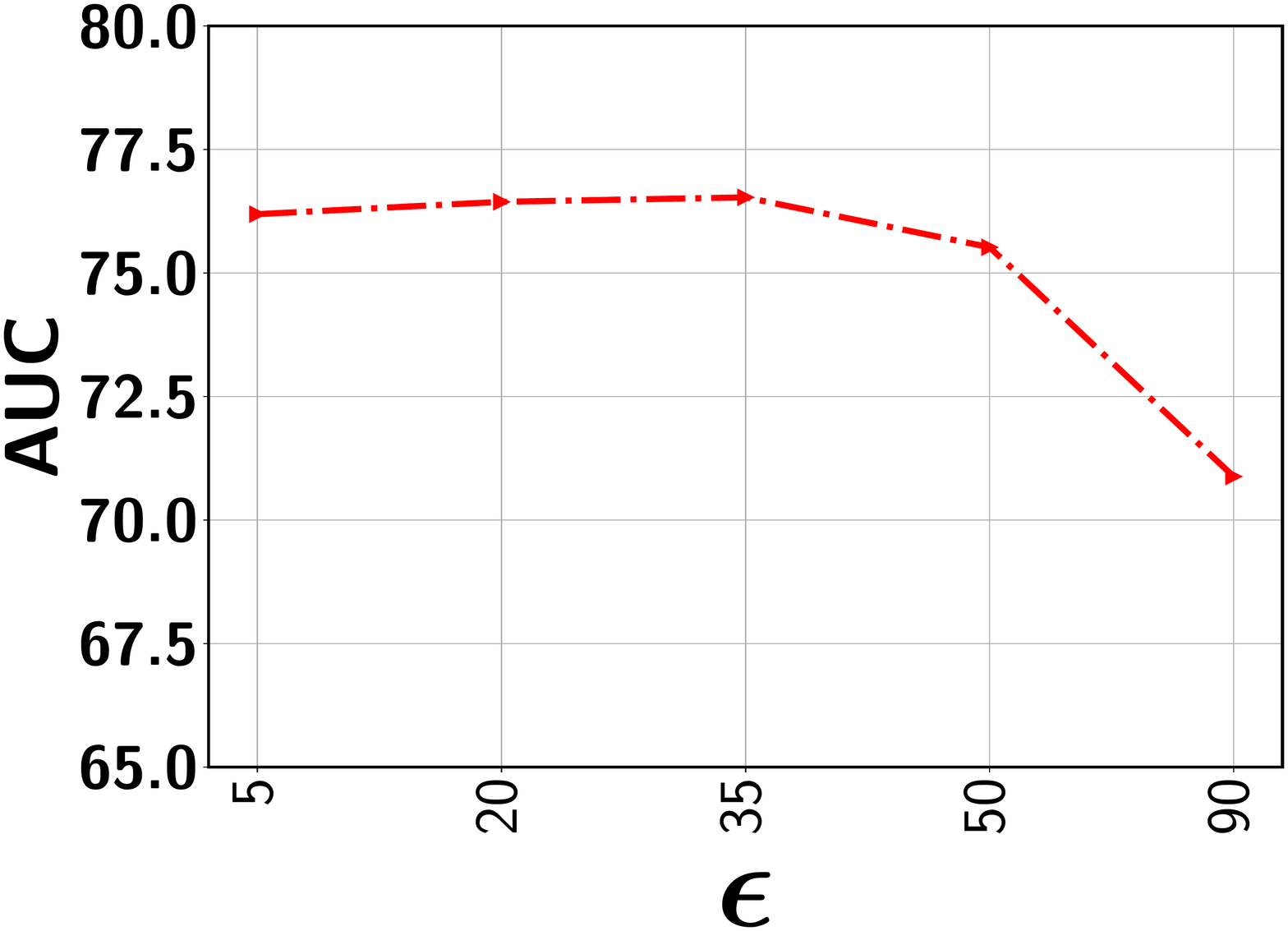}
  \vspace{-4mm}
  \caption{Multimodal (UCF)}
\end{subfigure}
\begin{subfigure}{0.19\textwidth}
  \centering
  \includegraphics[angle=90,width=\linewidth]{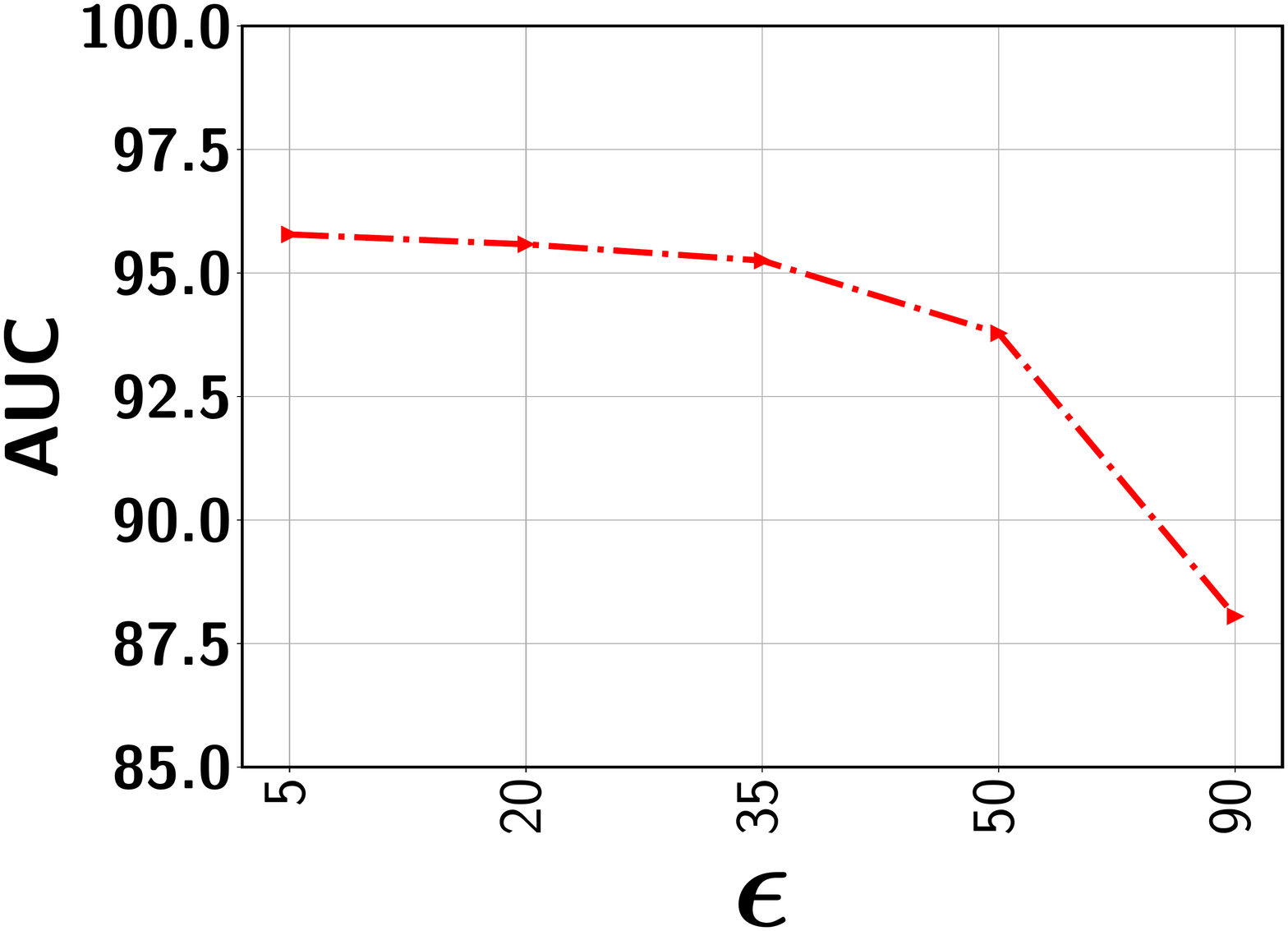}
  \vspace{-4mm}
  \caption{Outlier (ShanghaiTech)}
\end{subfigure}
\vspace{-2mm}
\caption{Performance variation with respect to $\epsilon$} 
\vspace{-0mm}
\label{fig: ablation_epsilon}
\end{figure*}
\paragraph{Impact of the constraint $|C|\leq \kappa$} It should be noted that in our approach  $\kappa$ only provides an upper bound on the selected segments and the actual number is determined by the non-parametric model along with the prediction threshold $\epsilon$. This addresses the fundamental issue in the top-$k$ models, in which a fixed $k$ has to be set for all videos. Figures~\ref{fig: performance_kappa} shows that a stable performance can be achieved for a wide range of $\kappa$ values as long as it is not set too small that may exclude some representative abnormal segments.

\begin{figure}[h!]
\vspace{-12mm}
\centering
\begin{subfigure}{0.19\textwidth}
  \centering
  \includegraphics[width=\linewidth]{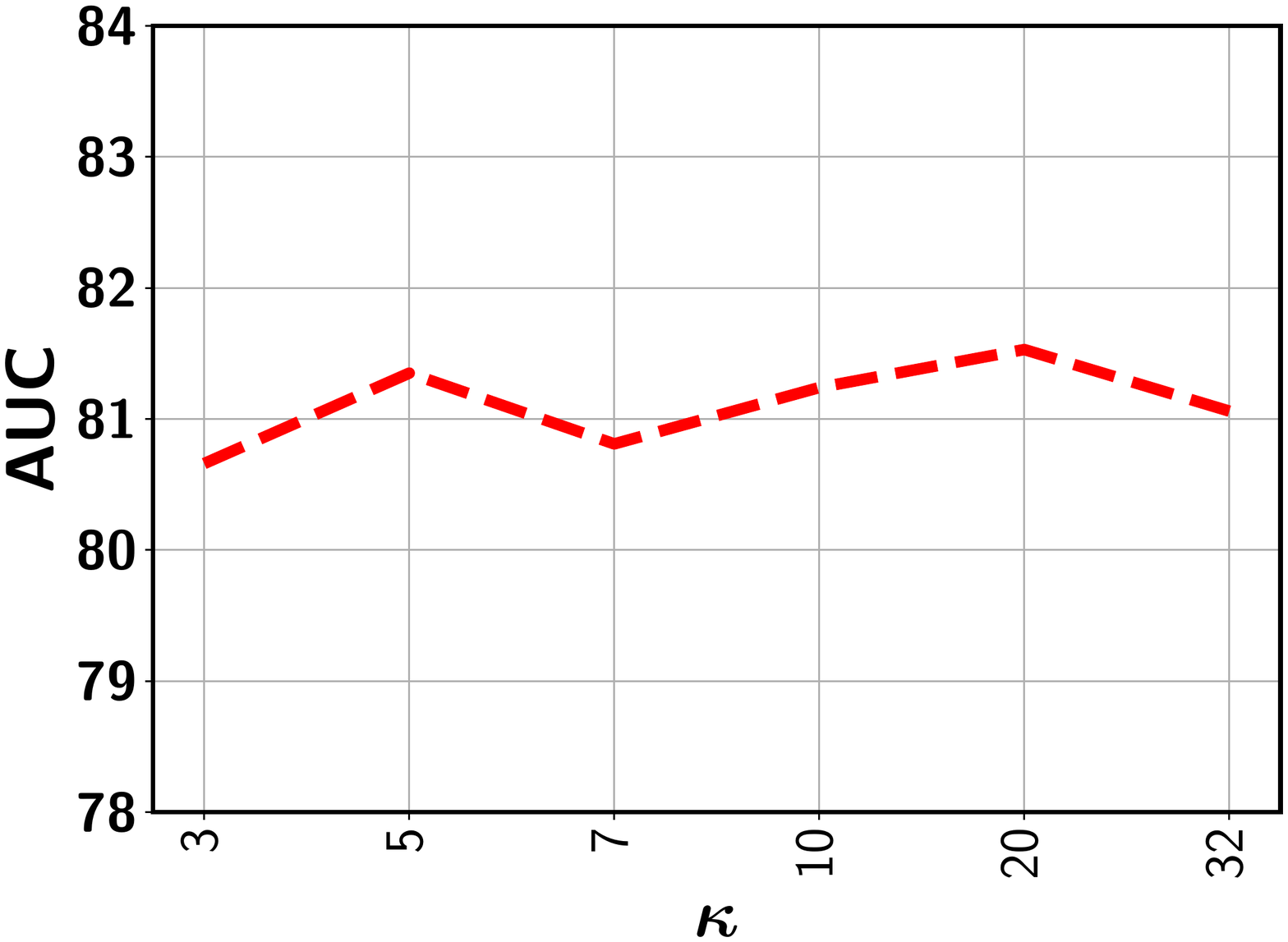}
  \vspace{-12mm}
  \caption{UCF-Crime}
\end{subfigure}%
\begin{subfigure}{0.19\textwidth}
  \centering
  \includegraphics[width=\linewidth]{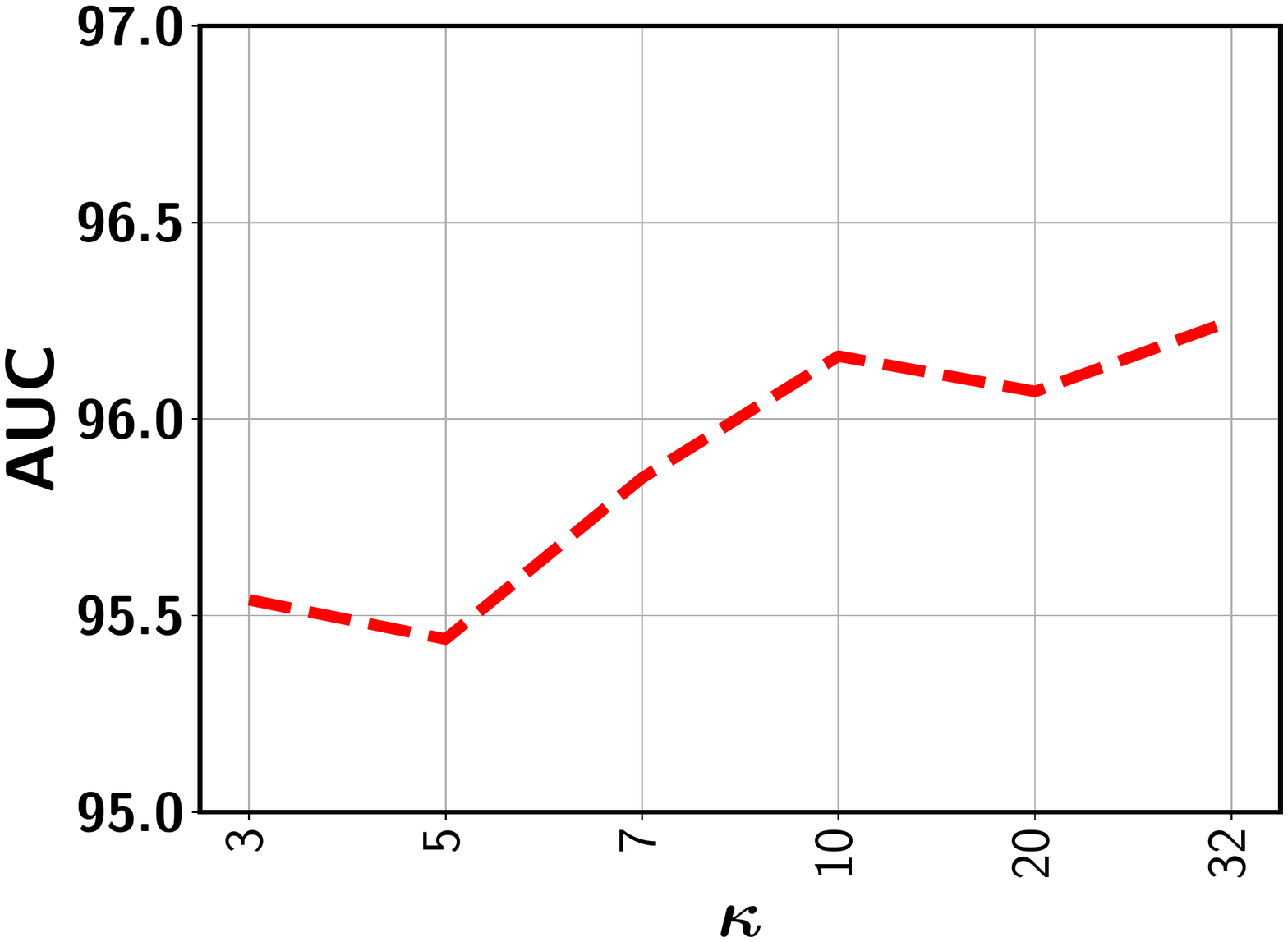}
   \vspace{-12mm}
\caption{ShanghaiTech}
\end{subfigure}%
\begin{subfigure}{.19\textwidth}
  \centering
  \includegraphics[width=\linewidth]{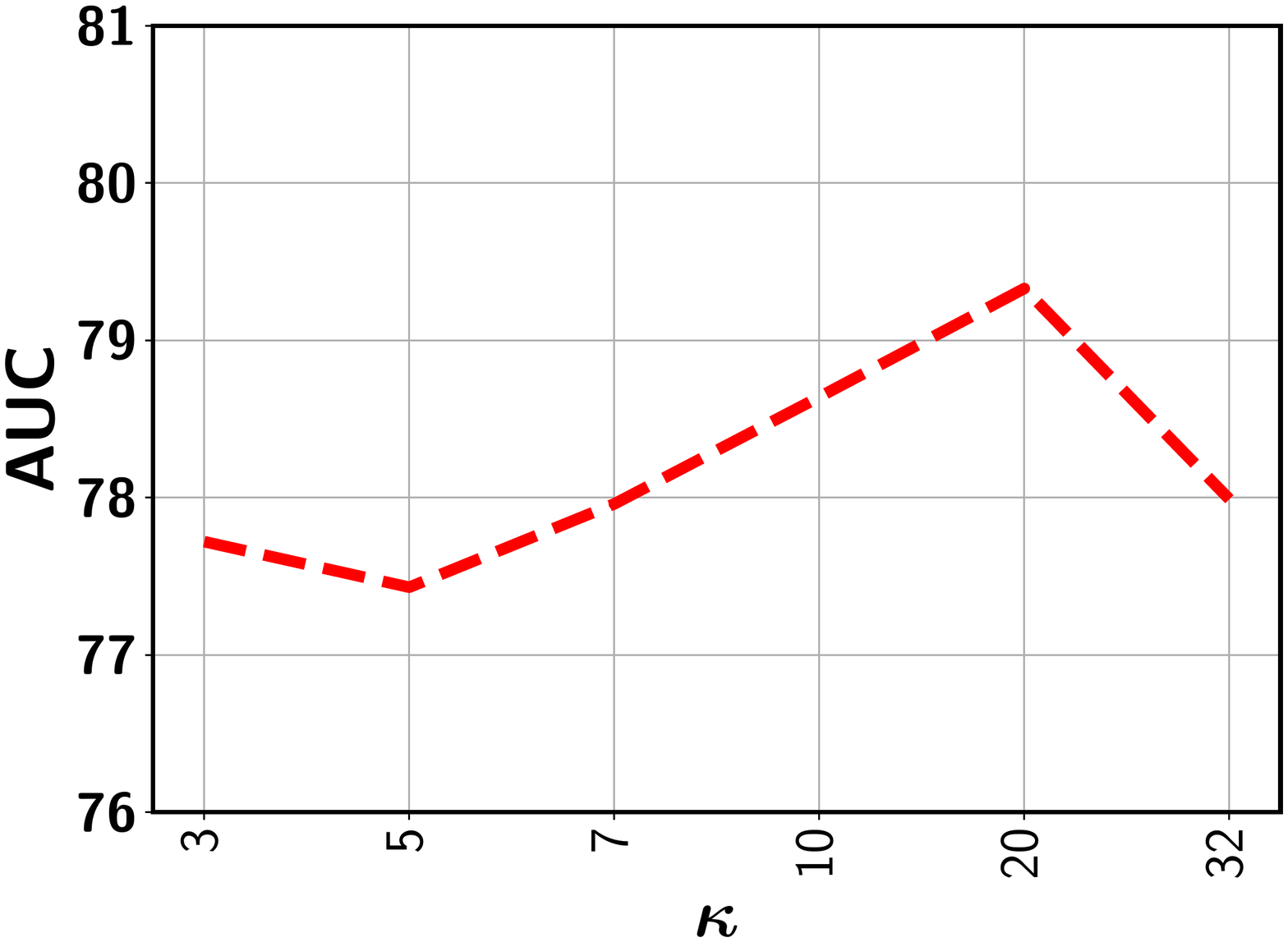}
   \vspace{-12mm}
\caption{Avenue}
\end{subfigure}%
\begin{subfigure}{.19\textwidth}
  \centering
  \includegraphics[width=\linewidth]{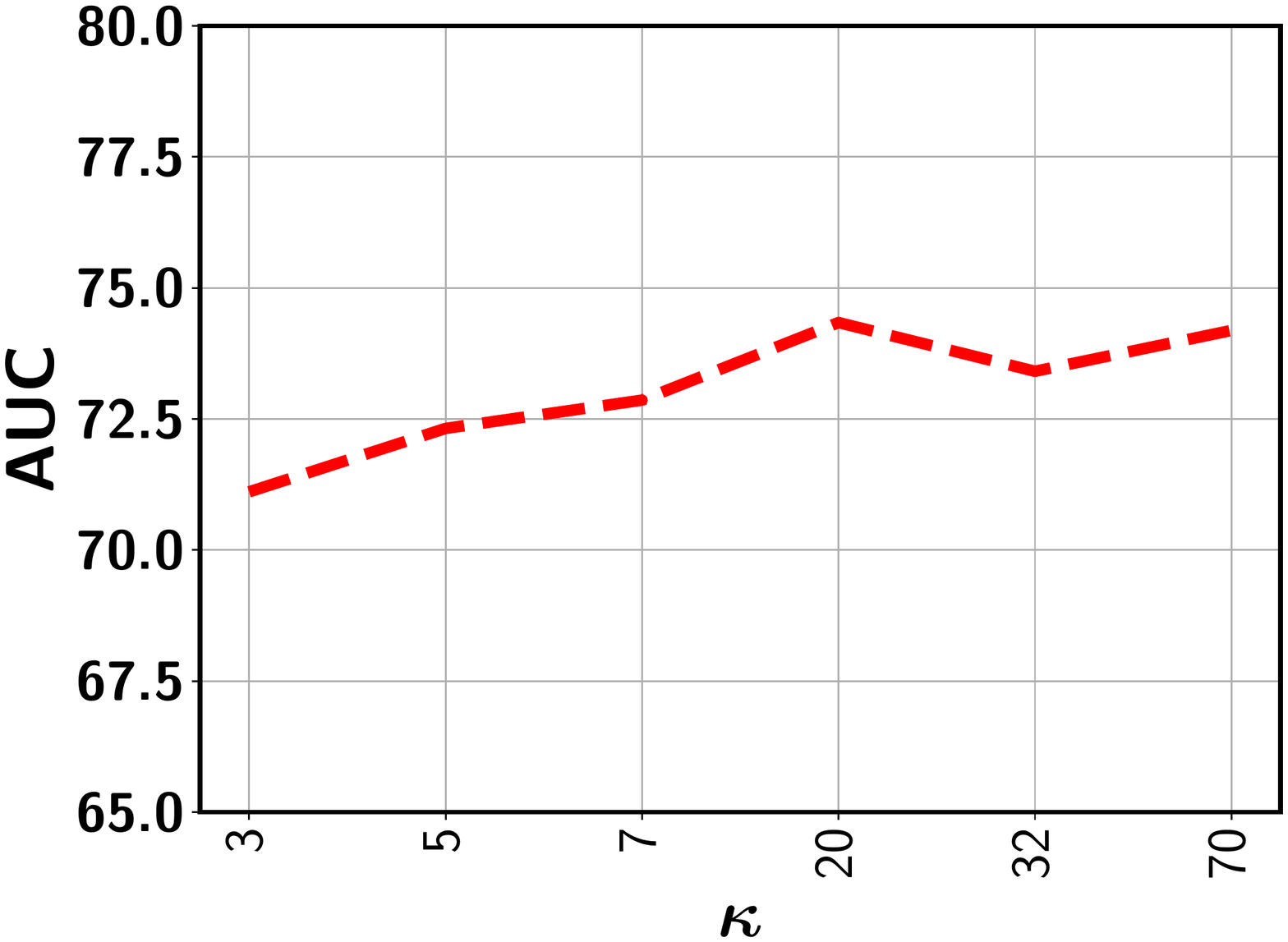}
   \vspace{-12mm}
\caption{Multimodal (UCF)}
\end{subfigure}%
\begin{subfigure}{.19\textwidth}
  \centering
  \includegraphics[width=\linewidth]{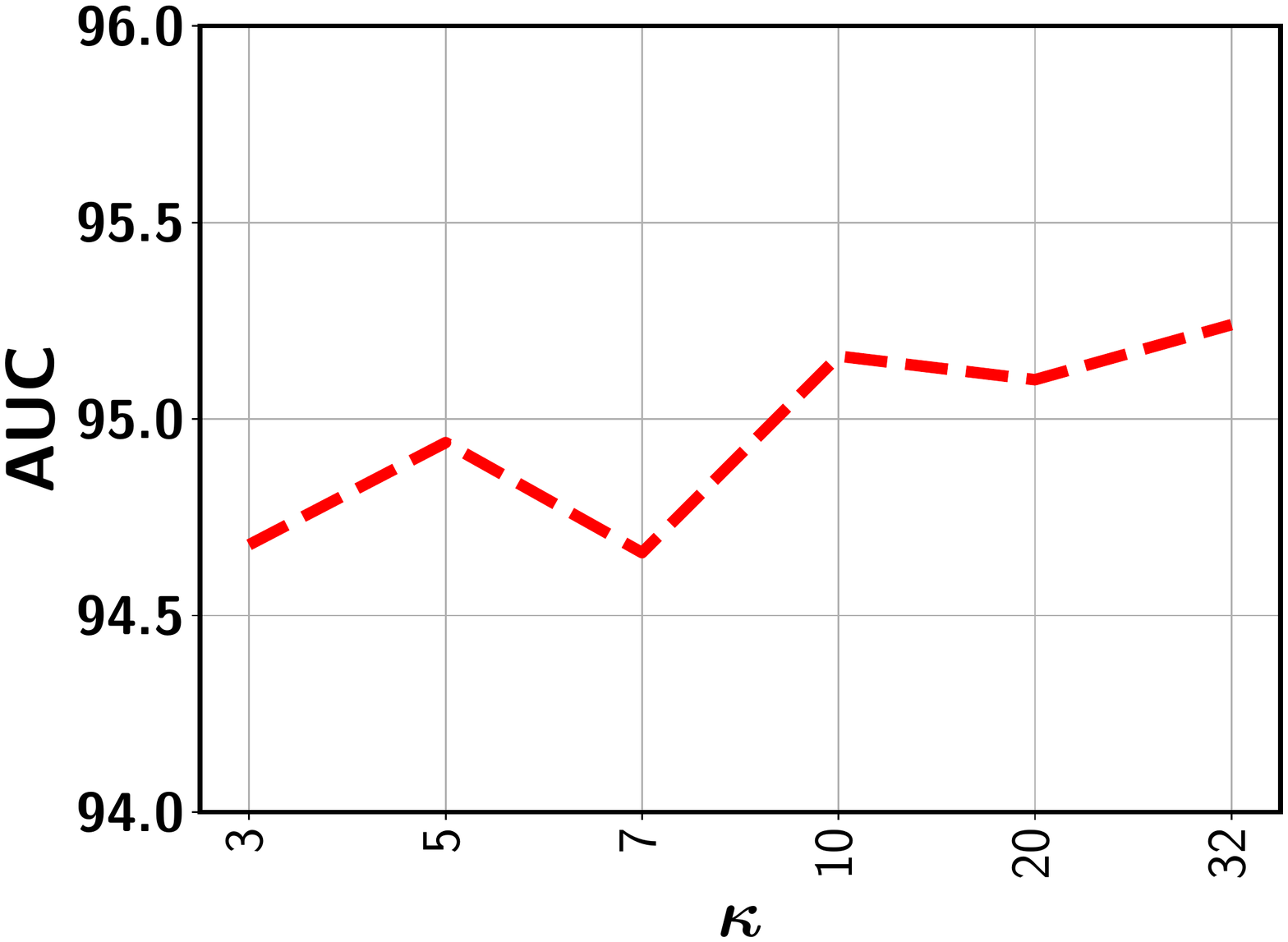}
   \vspace{-12mm}
\caption{Outlier (ShanghaiTech)}
\end{subfigure}

\caption{Performance variation with respect to $\kappa$}
\label{fig: performance_kappa}
\vspace{-5mm}
\end{figure} 

\paragraph{Impact of $\lambda$}
We would like to emphasize that BN-SVP does not require to directly set $\lambda$, which is highly challenging. 
By leveraging the prediction score of instances and their mixture assignments, BN-SVP implicitly sets $\lambda$ to balance MIL loss and the diversity of set $\widehat{\mathcal{C}^+}$. Specifically because of the constraint $|{\mathcal C}^+|\leq \kappa$, we ensure that the set contains no more than $\kappa$ segments. It excludes segments with a low prediction score, which has the effect of decreasing $\lambda$ to reduce the MIL loss. 
Similarly, instead of choosing the instance with the largest pairwise similarity with all other instances within the same mixture assignment, it chooses the instance with the highest prediction score. This can also be viewed as choosing a smaller $\lambda$ to reduce the MIL loss.

\begin{table}[h!]
\caption{Performance (AUROC) with and without augmentation}
\label{tab: ablation_rho}
\vspace{-4mm}
\begin{center}

\begin{tabular}{|c|c|c|c|c|c|}
\hline
\textbf{Dataset}&  \textbf{UCF-Crime}  & \textbf{Avenue}  &   \textbf{Multimodal} & \textbf{ShanghaiTech} & \textbf{Outlier} \\
\hline

w augmentation & $83.39$  & $80.87$    & $76.53$ & $96.00$ & $95.27$  \\
\hline
w/o augmentation & $80.56$  & $76.71$  & $63.23$ & $94.99$ & $94.52$ \\
\hline

\end{tabular}
\end{center}

\vspace{-6mm}
\end{table}

\paragraph{Impact of Augmentation} We compare the performance (AUROC) with augmentation ($\rho = 1$) and without augmentation ($\rho = 0$).
Table~\ref{tab: ablation_rho} shows the result for different datasets. As can be seen, BN-SVP consistently performs better on all datasets than w/o augmentation. 
Without augmentation, the approach transitions from one state to another state quickly for small visual changes and may not be able to keep the temporal persistence when discovering the scenes and therefore the performance is lower. We have also shown the significance of using augmentation via a qualitative analysis in Appendix~\ref{app:BNP}.

\subsection{Additional Qualitative Analysis}
To show the effectiveness of the proposed approach to handle multimodality, we compare BN-SVP with MMIL using some illustrative examples. Figure~\ref{fig: qualitative_avenue} shows two frames from the TEST06 video in Avenue with different anomaly types. In the first anomaly type, the object is thrown and in the second, a person is walking in the wrong lane. As the first anomaly is more obvious, both BN-SVP and MMIL are able to correctly predict it as abnormal. For the second one, our proposed approach correctly detects it as abnormal while MMIL fails to do that. Due to the submodular diversified loss, BN-SVP is more likely to include even less obvious frames (\eg, Figure \ref{fig: qualitative_avenue} (b)) during the training process and as a result the approach can make a correct prediction. On the other hand, MMIL picks the one with maximum score and therefore more likely to miss those less obvious ones during training process resulting in the mis-identification of similar frames as normal.

\begin{figure}[t!]
\centering
\begin{subfigure}{0.3\textwidth}
  \centering
  \includegraphics[width=.9\linewidth]{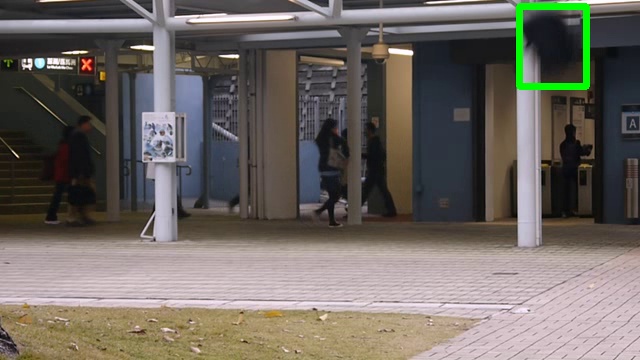}
  \caption{Frame 1}
\end{subfigure}
\begin{subfigure}{0.3\textwidth}
  \centering
  \includegraphics[width=.9\linewidth]{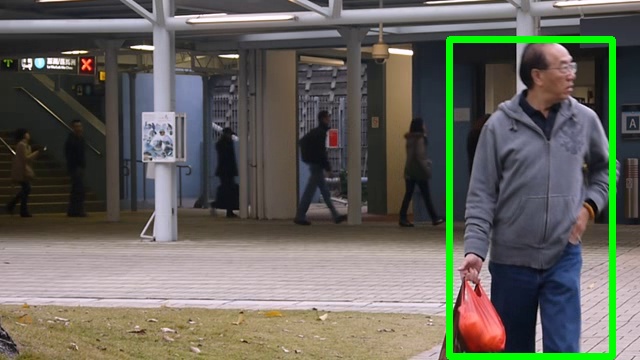}
  \caption{Frame 2}
\end{subfigure}
\vspace{-2mm}
\caption{Example frames from UCF-Crime Stealing019; (a) Correct BN-SVP, MMIL, (b) Correct BN-SVP, incorrect MMIL }
\label{fig: qualitative_avenue}
\vspace{-0mm}
\end{figure}

\subsection{Effectiveness of Bayesian Non-Parametric Video Partition}\label{app:BNP}
In this section, we present representative frames from the mixture components (\ie sub-scenes) discovered by the proposed Bayesian non-parametric video partition process. The purpose is to demonstrate that semantically coherent segments are automatically grouped into the same mixture components by the proposed BN-SVP. This significantly facilitates the optimization of the submodular function to choose a diverse set of segments and allow some of the most representative segments to participate in the MIL loss. 
Figure \ref{fig: effectiveness_hdp_hmm} shows frames randomly selected from different mixture components for video {\sc 01\_0162} from the ShanghaiTech dataset. As shown in Figure \ref{fig: effectiveness_hdp_hmm} (a),  the frame does not contain any person and its associated component (\ie, Component 0) mostly consists of background segments (which are predicted as normal by the model). 
In Figure \ref{fig: effectiveness_hdp_hmm} (b), there are multiple people in the frame. Furthermore, someone is riding a bike in a wrong lane while a second person is pointing to another group of people. This frame is assigned to a newly created component (\ie, Component 1) since it looks quite different from the previous frames. Also, given the abnormal behavior in the frame, the model predicts it as an anomaly. 
Next shown in Figure \ref{fig: effectiveness_hdp_hmm} (c), as the bike starts to vanish from the camera frame, it looks different from (b) and therefore the model assigns it to a new Component 2. Although (b) and (c)  are both of abnormal types, the latter is much less obvious than the former. Given their distinctions, they have been assigned to different mixture components so both of them could be chosen when optimizing the submodular function to participate in model training. 
Finally, for Figure \ref{fig: effectiveness_hdp_hmm} (d), the bike completely disappears from the frame and only a group people walking normally. So, it is assigned to Component 3 and the model predicts its as normal.  

\begin{figure*}[t!]
\centering
\begin{subfigure}{0.23\textwidth}
  \centering
  \includegraphics[width=\linewidth]{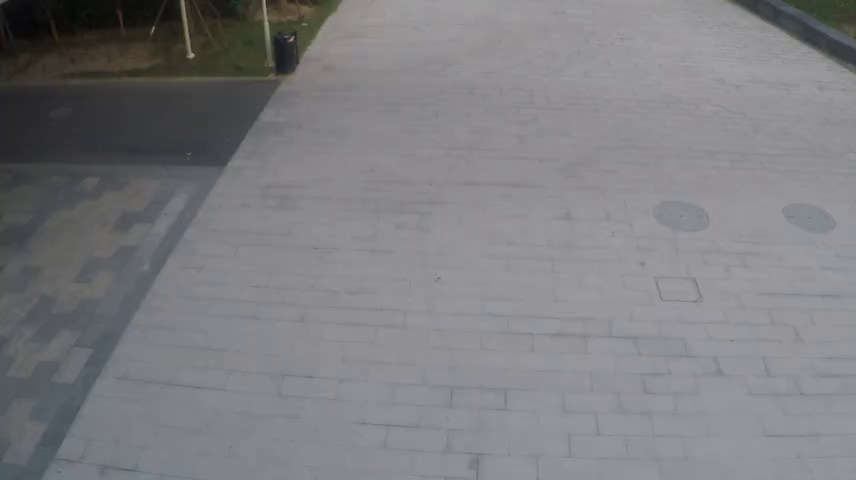}
  \vspace{-4mm}
  \caption{Normal Frame (Component 0)}
\end{subfigure}
\begin{subfigure}{.23\textwidth}
  \centering
  \includegraphics[width=\linewidth]{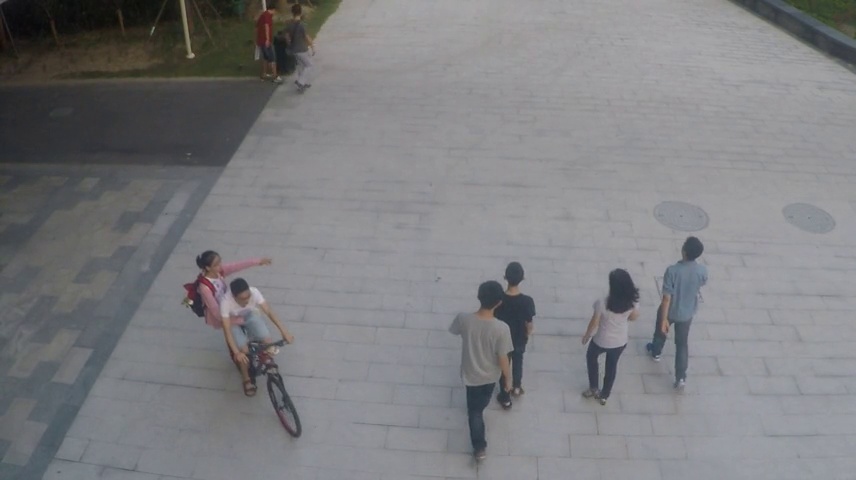}
  \vspace{-4mm}
  \caption{Abnormal Frame (Component 1)}
\end{subfigure}%
\begin{subfigure}{0.23\textwidth}
  \centering
  \includegraphics[width=\linewidth]{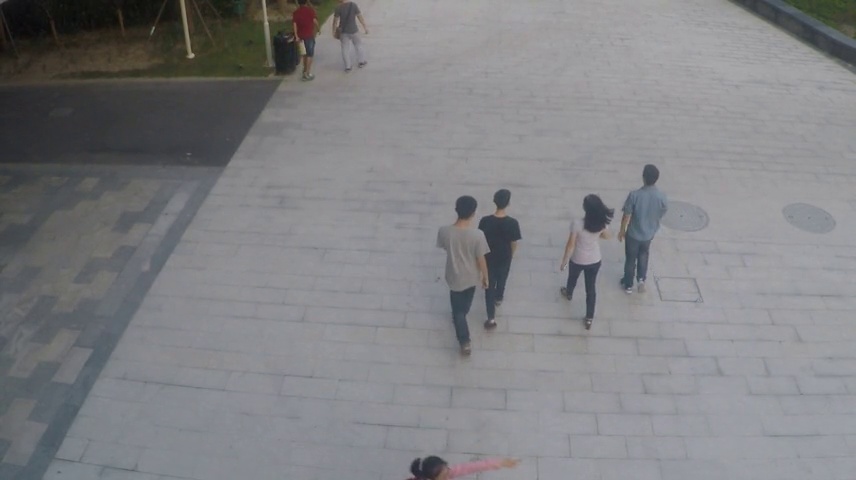}
  \vspace{-4mm}
  \caption{Abnormal Frame (Component 2)}
\end{subfigure}
\begin{subfigure}{0.23\textwidth}
  \centering
  \includegraphics[width=\linewidth]{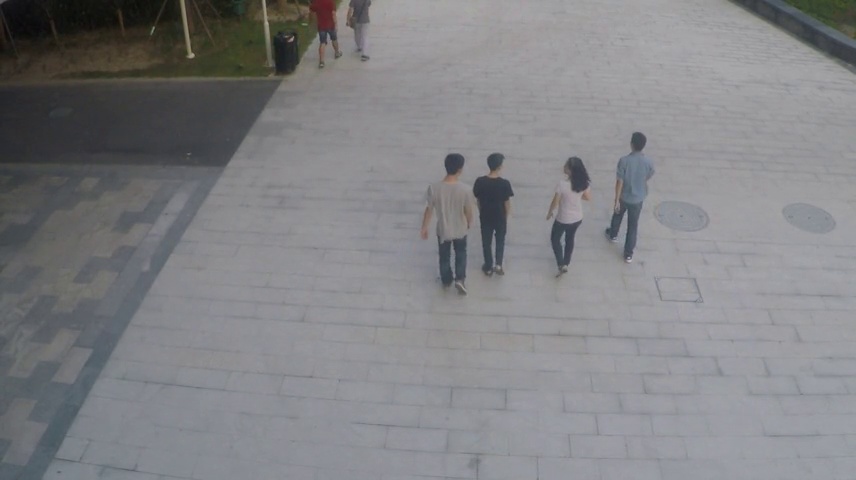}
  \vspace{-4mm}
  \caption{Normal Frame (Component 3)}
\end{subfigure}

\vspace{-2mm}
\caption{Example frames from the discovered mixture components} 
\vspace{-0mm}
\label{fig: effectiveness_hdp_hmm}
\end{figure*}

\section{Link to Source Code}
\label{sec:link_source_code}
For the source code, please click the following link: \url{https://github.com/ritmininglab/BN-SVP}.

\end{document}